\documentclass[10pt]{article}

\usepackage{amsmath,amsthm,verbatim,amssymb,amsfonts,amscd, graphicx}
\usepackage{graphics}
\usepackage{centernot}
\usepackage{authblk}
\usepackage{bbm}

\theoremstyle{plain}
\newtheorem{theorem}{Theorem}

\newtheorem{proposition}{Proposition}

\newtheorem{definition}{Definition}

\newcommand{\p}{\mathbb{P}}

\newcommand{\weight}{\theta}

\usepackage{paralist}

\usepackage[colorlinks,linkcolor=blue,citecolor=blue]{hyperref}

\usepackage[bottom]{footmisc}
\usepackage{caption}
\usepackage{subcaption}
\usepackage{helvet}  
\usepackage{courier}  
\usepackage{url}  
\usepackage{graphicx}  
\usepackage{multirow}
\usepackage{amsthm}
\usepackage{color}
\usepackage{MnSymbol}
\usepackage{makecell}
\usepackage{arydshln}
\usepackage{amsmath}
\usepackage[dvipsnames]{xcolor}
\usepackage{caption} 
\usepackage[numbers]{natbib}
\usepackage{textcomp}
\usepackage{wrapfig}
\usepackage{algorithm}
\usepackage{algorithmic}

\usepackage{csquotes}

\newcommand{\E}{\mathbb{E}}

\title{Type-II Saddles and Probabilistic Stability of Stochastic Gradient Descent}

\usepackage{enumitem}

\begin{document}
\author{Liu Ziyin$^{1,2}$, Botao Li$^{3}$, Tomer Galanti$^4$, Masahito Ueda$^{5,6}$\\

$^1$\textit{Research Laboratory of Electronics, Massachusetts Institute of Technology}\\
$^2$\textit{Physics and Informatics Laboratories, NTT Research}\\
\textit{$^3$Laboratoire de Probabilit\'es, Statistique et Mod\'elisation/Universit\'e Paris Cit\'e}\\
\textit{$^4$ Center of Brains, Minds and Machines (CBMM), Massachusetts Institute of Technology}\\
\textit{$^5$RIKEN Center for Emergent Matter Science (CEMS)}\\
\textit{$^6$Department of Physics, The University of Tokyo}}
\maketitle

\begin{abstract}
Characterizing and understanding the dynamics of stochastic gradient descent (SGD) around saddle points remains an open problem. We first show that saddle points in neural networks can be divided into two types, among which the Type-II saddles are especially difficult to escape from because the gradient noise vanishes at the saddle. The dynamics of SGD around these saddles are thus to leading order described by a random matrix product process, and it is thus natural to study the dynamics of SGD around these saddles using the notion of {probabilistic stability} and the related Lyapunov exponent. Theoretically, we link the study of SGD dynamics to well-known concepts in ergodic theory, which we leverage to show that saddle points can be either attractive or repulsive for SGD, and its dynamics can be classified into four different phases, depending on the signal-to-noise ratio in the gradient close to the saddle.
\end{abstract}

\section{Introduction}
Saddle points exist extensively in the loss function of neural networks. An important problem in deep learning theory is to characterize and understand the dynamics of SGD around saddle points. In this work, we divide the saddle point problem into two types: (1) Type-I, where the gradient noise is nonvanishing at the saddle, and (2) Type-II, where there is no gradient noise at the saddle. The main focus of our work is the stability of the Type-II saddles, which are more difficult to escape from using SGD. We show that these saddle points can be either attractive or repulsive for SGD, even if the Hessian matrix contains at least one negative eigenvalue. Recent progress in the learning dynamics of neural networks suggests that the trajectory of SGD is predominantly saddle-to-saddle \cite{jacot2021saddle, abbe2023sgd}. This means that SGD will stay close to many saddle points during training and makes it possible for SGD to actually ``stop" at one of these saddle points at the end of training. It is thus important and urgent to understand the stability\footnote{Here, ``stability" is a synonym of ``attractivity" and should not be confused with the algorithmic stability.}  of these saddle points.

Our main contributions are:
\begin{enumerate}[leftmargin=15pt, noitemsep,topsep=0pt, parsep=1pt,partopsep=1pt]
    \item we propose to classify saddle points in neural networks into two types, among which the Type-II saddles are shown to be difficult to escape;
    \item we propose using the probabilistic stability and Lyapunov exponents to study the attractivity of Type-II saddle points of neural networks; this bridges the conventional study of dynamical systems in control theory and ergodic theory to study the dynamics of SGD;
    \item we use the probabilistic stability to show that close to a Type-II saddle point, SGD has at least four different phases of learning, which we find to be relevant for understanding the initialization of neural nets.
\end{enumerate}
This work is organized as follows. Section~\ref{sec: related works} discusses the most relevant works. Section~\ref{sec: two types of saddles} discusses the difference between the two types of saddle points encountered in training neural networks. Section~\ref{sec: main results} studies the probabilistic stability and Lyapunov exponent of SGD around these fixed points and connects to the well-known results in ergodic theory. Section~\ref{sec: exp} presents the experiments. All proofs and experimental details are presented in the Appendix.

\section{Related Works}\label{sec: related works}
\textbf{Dynamical stability}. Dynamical stability centers around studying the attractivity and repulsiveness of fixed points. From an optimization perspective, stability can be seen as a worst-case guarantee for optimization because the training algorithm cannot converge to a fixed point that is unstable. Conceptually, the work closest to ours is Ref.~\cite{wu2018sgd}, which uses the variance of SGD to characterize the set of solutions that are preferred by SGD, namely the solutions that are flat and have a rather weak noise. However, as our result shows, moment-based notions of stability cannot be applied to saddle points. Instead, we propose to study this problem with the tools of probabilistic stability, which has been studied extensively in control theory \cite{khas1967necessary, eckmann1985ergodic, teel2014converse}. The connection between Lyapunov exponents and probabilistic stability for the random matrix product problem is well known \cite{diaconis1999iterated}, but its relevance for studying saddle points in deep learning is unclear because one rarely defines a ``saddle point" for matrix product problems. Also, its relevance to deep learning has yet to be studied. Prior works have used similar notions to study the convergence to a local minimum \cite{gurbuzbalaban2021heavy, hodgkinson2021multiplicative}. In contrast, our focus is on the stability and attractivity of saddle points.

\begin{wrapfigure}{r}{0.4\linewidth}
\includegraphics[width=0.9\linewidth]{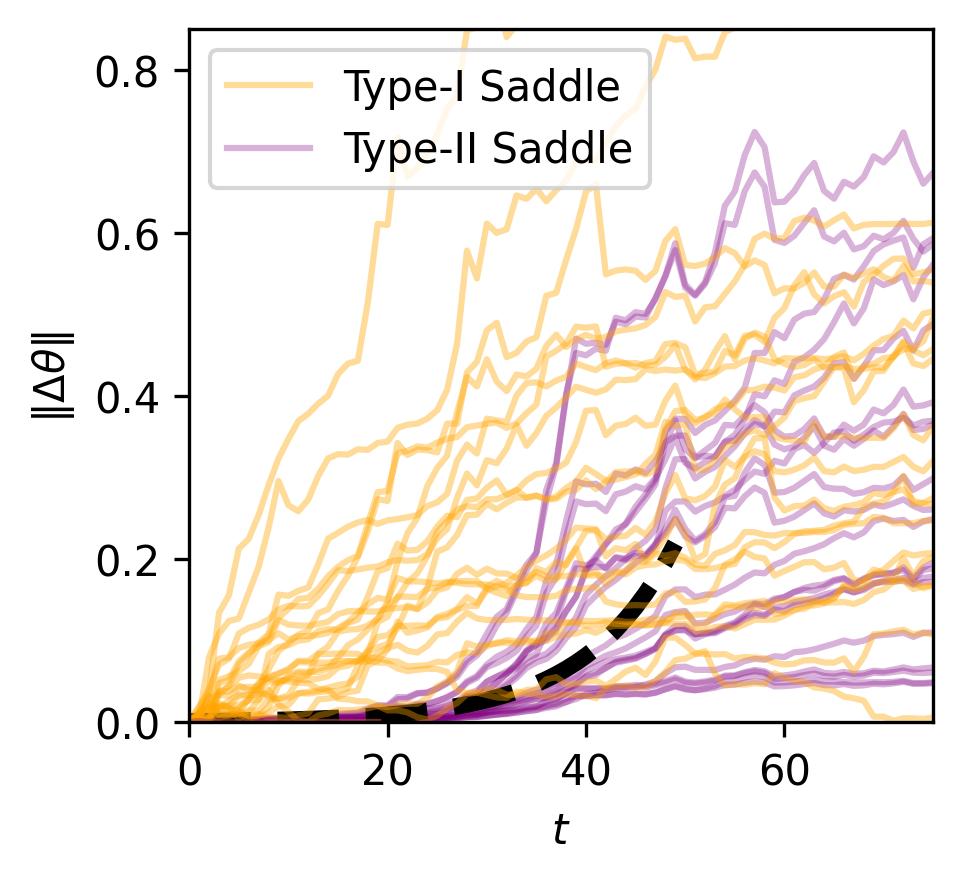}
\vspace{-1em}
\caption{\small Escaping from two types of saddles in a ReLU network under SGD. We see that for the escapes from type-I saddles, the escaping process starts immediately for every trajectory. For type-II saddles, the escape only starts significantly after the training starts, despite the gradient noise. The black dashed line shows an exponential fit to the type-II, implying the connection of the problem to Lyapunov exponents. See Appendix~\ref{app sec: figure 1 setting} for details on the construction of these saddles.}\label{fig: two types of saddles}
\vspace{-2em}
\end{wrapfigure}

\textbf{Saddle points}. Due to the nonlinear nature of the loss function of neural networks, escaping from saddle points has been an important problem in deep learning \cite{dauphin2014identifying, Ge2015, Mertikopoulos20, Vlaski22, Pemantle1990}. As a primary example, Ref.~\cite{dauphin2014identifying} mainly focuses on studying saddles from the perspective of gradient descent, where there is no difference between Type-I and Type-II saddles. Also, our focus is not on how to escape from saddles, but on the scientific question of whether a saddle point is attractive or not. See the next section for more discussion about prior literature on saddles.

\section{Two Types of Saddle Points}\label{sec: two types of saddles}

Escaping from saddle points is a well-known and fundamental problem in the optimization of neural networks. Conventional wisdom believes that as long as the saddle points are strict (having at least one negative eigenvalue in the Hessian), escaping from it is easy \cite{sun2020global}. This holds true for gradient descent training. However, it is unclear whether strictness is sufficient to guarantee an escape when the training proceeds with SGD, where the gradient noise confounds the training process.

Here, we propose to classify the saddle points into two types and show that the Type-II saddles are much more difficult to escape than the Type-I saddles. A point $\theta$ is said to be a saddle point when $\E[\nabla_\theta\ell] =0$ and $\theta$ is not a local minimum. This implies that 
 \begin{equation}
     \E[\theta_{t+1} -\theta_t] = O(\|\Delta \theta\|),
 \end{equation}
where $\|\Delta \theta\|$ is the Euclidean distance from the saddle and $t$ is the parameter at the $t$-th iteration of SGD. However, when stochasticity is taken into account, there are two types of saddles that satisfy this condition:
\begin{equation}
    \theta_{t+1} -\theta_t = \lambda r_t + O(\|\Delta \theta\|) \quad \text{(Type-I saddle)};
\end{equation}
\begin{equation}\label{eq: type two saddle dynamics}
    \theta_{t+1} -\theta_t = \lambda \hat{H} \Delta \theta_t + O(\|\Delta \theta\|^2) \quad\text{(Type-II saddle)},
\end{equation}
where $\lambda$ is the learning rate, $r_t$ is a random vector with a non-vanishing variance that is independent of $\theta$, and $\hat{H}$ is a random matrix, also independent of theta. The source of the randomness is due to the minibatch sampling over the data, and $\hat{H}$ can be easily identified as the sample-wise Hessian matrix. There is a crucial difference between the two types of saddles: Type-I saddles have a nonvanishing noise at the saddle, and this noise makes it possible for SGD to escape from the saddle. Meanwhile, the Type-II saddles have a vanishing noise at the saddle, making it more difficult to escape from. See Figure~\ref{fig: two types of saddles} for a comparison of escaping from the two types of saddles during neural network training. Here, the Type-II saddle is induced by the permutation symmetry between neurons in the hidden layer. This discussion motivates the following definition for Type-II saddles. Let $H(\theta)$ denote the Hessian of the loss at $\theta$, and $P$ denote the projection to the subspaces where $H$ has non-positive eigenvalues.
\begin{definition}
    $\theta$ is a Type-II saddle of $\ell(\theta,x)$ if (1) $ P \nabla_\theta \ell(\theta,x) = 0$ for all $x$ and (2) $\theta$ is not a local minimum of $\E[\ell(\theta,x)].$
\end{definition}
Setting $P=I$ recovers Eq.~\eqref{eq: type two saddle dynamics} and so this definition generalizes Eq.~\eqref{eq: type two saddle dynamics}. Prior works suggest that Type-II saddle points exist abundantly in the loss function of neural networks and the matrix $P$ depends on the architecture of the network. The main cause of the saddle points in neural networks is the vast number of parameter symmetries, such as permutation symmetry \cite{fukumizu2000local, simsek2021geometry, entezari2021role, hou2019minimal}, rescaling symmetry \cite{Dinh_SharpMinima, neyshabur2014search}, and rotation symmetry \cite{ziyin2023what}. 
More recently, Ref.~\cite{ziyin2023symmetry} proves that close to these symmetry-induced saddle points, the SGD dynamics is always described by the type of dynamics in Eq.~\eqref{eq: sgd dynamics linearized}. Convergence to these symmetry-induced saddles has been identified as a main cause of collapsing of models into low-capacity solutions \cite{ziyin2023symmetry}. Literature regarding SGD escaping saddle points  (e.g., \cite{Ge2015, Mertikopoulos20, Vlaski22, Pemantle1990}) often takes advantage of the assumption of non-vanishing noise thus studies escaping type-I saddle. Thus, Type-II saddles are understudied even despite their relevance in deep learning.

A main result in \cite{ziyin2023symmetry} shows that as long as the saddle is symmetry-induced, all escaping directions are contained within the subspace given by a projection matrix $P$, and the dynamics \eqref{eq: type two saddle dynamics} take a low-rank form:
\begin{equation}\label{eq: sgd dynamics linearized}
    P\theta_{t+1} = P\theta_t - \lambda  P^T\hat{H}(x_t) P (\theta_t - \theta^*), 
\end{equation}
where $\lambda$ is the learning rate, $\theta^*$ is the critical point under consideration, $\hat{H}(x_t)$ is a random symmetric matrix that is a function of the random variables $x_t$, which can stand for both a single data point or a minibatch of data points. 
Essentially, it is the data distribution of $\hat{H}$ that matters. When we have the same data distribution but a different batch size $S$, the distribution of $\hat{H}$ is different. For our purpose, it suffices to say that it is interesting and important to study the case when $\hat{H}$ has a certain low-rank structure with probability $1$, which is frequently the case in neural networks. 

Lastly, while our focus is on the saddles, our theory also applies to interpolation minima in neural networks. Here, because the loss function for every data point vanishes, and the dynamics around it also obeys an identical form to Eq.~\eqref{eq: type two saddle dynamics}. The dynamics we consider is more general than that around an interpolation minimum because the Hessians $\hat{H}$ in Eq.~\eqref{eq: sgd dynamics linearized} are allowed to have both positive and negative eigenvalues, whereas an interpolation minimum only allows for positive semidefinite $\hat{H}$. Therefore, the general solution of Eq.~\eqref{eq: sgd dynamics linearized} also helps us understand this type of minima better once we restrict the study to PSD Hessians.

\section{Probabilistic Stability of SGD}\label{sec: main results}
In this section, we introduce the main theoretical framework. We first define the probabilistic stability, and then apply it to Type-II saddles.

\subsection{Probabilistic Stability}

We use $\|\cdot \|_{{p}}$ to denote the ${p}-$norm of a matrix or vector and $\|\cdot \|$ to denote the case where ${p}=2$. The notation $\to_p$ indicates convergence in probability.

\begin{definition}
    A sequence (of random variables) $\{
    \theta_t\}_{t}^{\infty}$ generated by the SGD algorithm is probabilistically stable at {a constant vector} $\theta^*$ if $\theta_t \to_p \theta^*$.
\end{definition}
Here, the notation $\to_p$ denotes convergence in probability. A sequence ${\theta}_t$ converges in probability to $\theta^*$ if $\lim_{t\to\infty} \p(\|\theta_t-\theta^*\| {>} \epsilon) = 0$ for any $\epsilon>0$. While this notion of stability is a simple use of standard probability theory, it is a common tool in control theory to study stability. We will see that this notion of stability is especially suitable for studying the stability of the {saddle} points in SGD. The standard alternative is the notion of attractivity based on statistical moments.

\begin{definition}
\label{def:2}
    A sequence $\{\theta_t\}_{t}^{\infty}$ is $L_p$-norm stable at $\theta^*$ if $\lim_{t\to \infty} \E\|\theta_t - \theta^*\|_p^p \to 0$.
\end{definition}

For deep learning, the sequence of $\theta_t$ is the model parameters obtained by the iterations of the SGD algorithm for a neural network. For dynamical systems in general and deep learning specifically, it is impossible to analyze the convergence behavior of the dynamics starting from an arbitrary initial condition. Therefore, we have to restrict ourselves to the neighborhood of a given stationary point and consider the linearized dynamics around it.

\begin{figure*}
    \centering
    \includegraphics[width=0.38\linewidth]{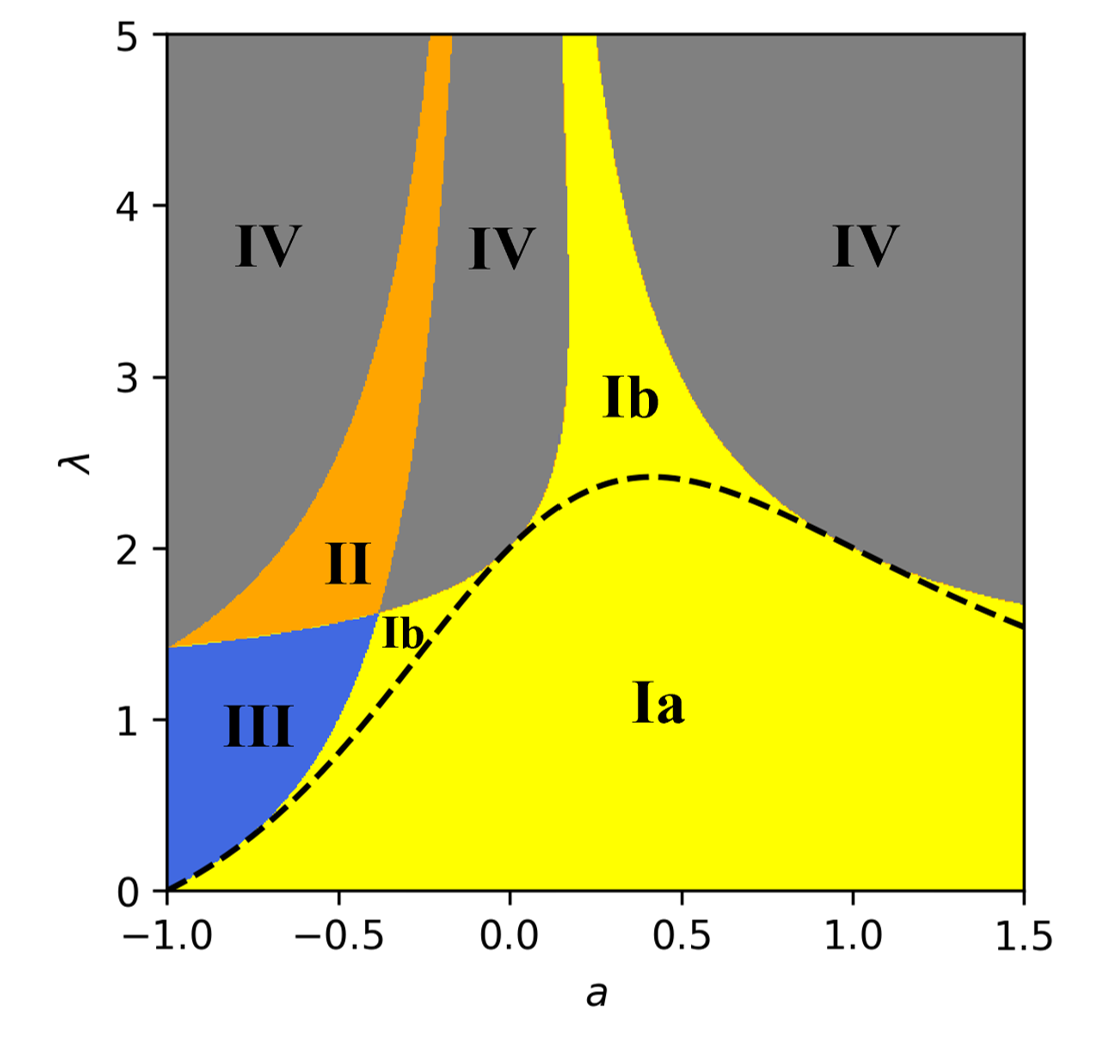}
    \includegraphics[width=0.42\linewidth]{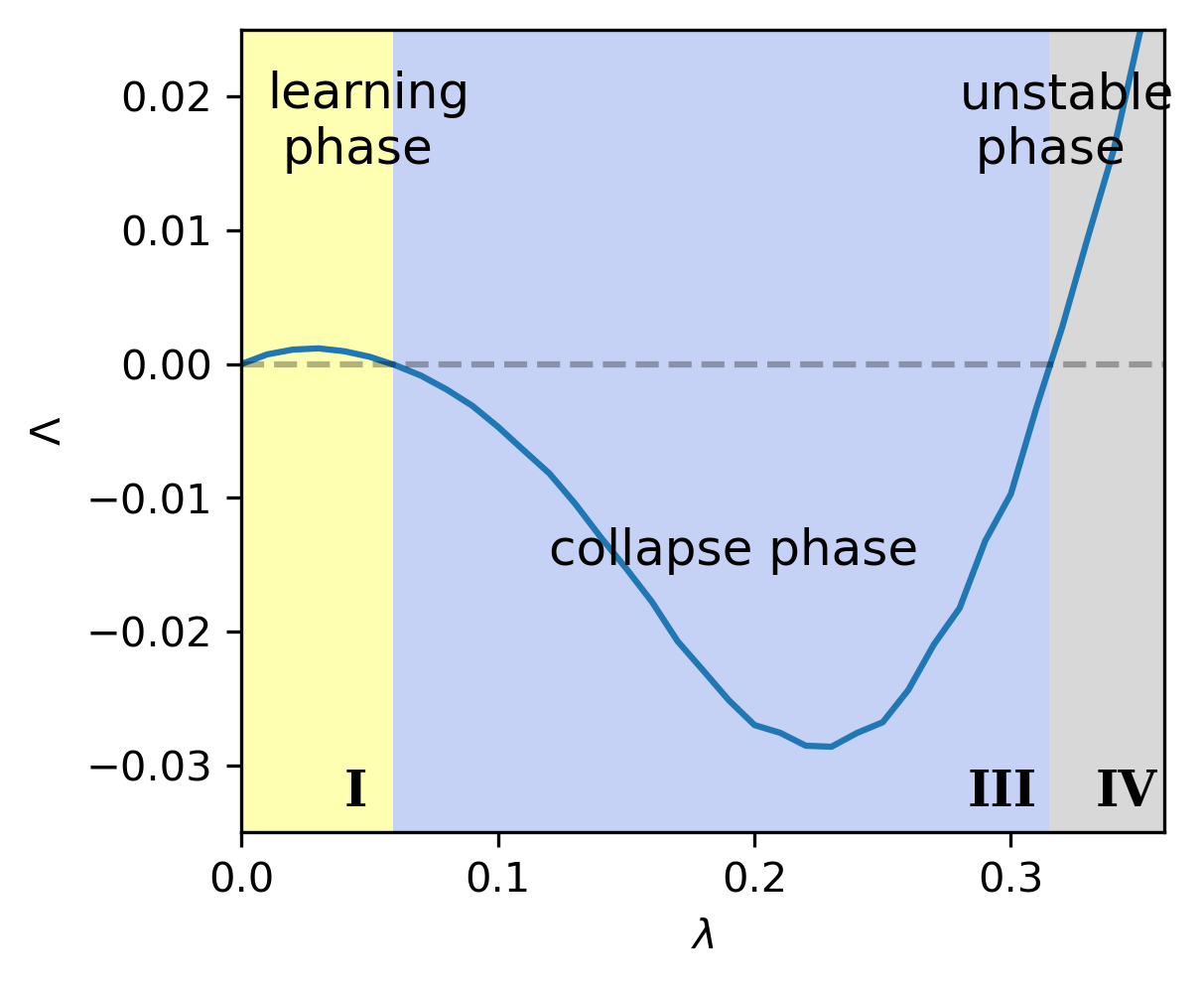}
    \caption{\small SGD exhibits a complex phase diagram through the lens of probabilistic stability. \textbf{Left}: {$a$ denotes the parameter in the data distribution, as discussed in detail in section \ref{sec: phases of learning sgd}.} For a matrix factorization saddle point, the dynamics of SGD can be categorized into at least five different phases. Phase \textbf{I}, \textbf{II}, and \textbf{IV} correspond to a successful escape from the saddle. Phase \textbf{III} is where the model converges to a low-rank saddle point. Phase \textbf{I} corresponds to the case $w_t \to_p u_t$, which signals correct learning. In phase \textbf{I{a}}, the model also converges in variance. Phase \textbf{II} corresponds to stable but incorrect learning, where $w_t \to_p -u_t$. Phase \textbf{IV} corresponds to complete instability. \textbf{Right}: the phases of SGD can quantified by the sign of the Lyapunov exponent $\Lambda$. Where $\Lambda<0$, SGD collapses to a saddle point; when $\Lambda >0$, SGD escapes the saddle and enters an escaping phase. The two escaping phases are qualitatively different. For a small learning rate, the model is in a learning phase due to the repulsiveness of the saddle point at a small learning rate, and the model is likely to converge to local minima close to the saddle. For a very large learning rate, SGD escapes the saddle due to the dynamical instability of SGD, and the model will move far away from the saddle. Besides, the magnitude of the Lyapunov exponent can also quantity the speed of the learning dynamics. See Appendix~\ref{app sec: exp detail} for numerical details of this example.}
    \label{fig:first phase diagram}
\end{figure*}

\subsection{Rank-1 Dynamics}
An analytically solvable case is when the dynamics lie in a 1d manifold. Let $\hat{H}(x)=h(x) nn^T$ is rank-$1$ for a random scalar function $h(x)$, and a fixed unit vector $n$ for all data points $x$. Thus, the dynamics simplifies to a 1d dynamics, where $h(x) \in \mathbb{R}$ is the corresponding eigenvalue of $\hat{H}(x)$:
\begin{equation}\label{eq: 1d symmetry dynamics}
    \theta_{t+1} = \theta_t - \lambda h(x) (\theta_t - \theta^*).
\end{equation}

\begin{theorem}\label{theo: 1d symmetry lyapunov condition}
    Let $\theta_t$ follow Eq.~\eqref{eq: 1d symmetry dynamics}. Then, for any distribution of $h(x )$, $n^T(\theta_t - \theta^*) \to_p 0$ if and only if
    \begin{equation}\label{eq: 1d prob stability condition}
        \E_x[\log |1-\lambda h(x)|] < 0.
    \end{equation}
\end{theorem}

The l.h.s. of the inequality can be identified as the Lyapunov exponent for the process. The condition \eqref{eq: 1d prob stability condition} is a sharp characterization of when a critical point becomes attractive. It also works with weight decay. When weight decay is present, the diagonal terms of $\hat{H}$ are shifted by $\gamma$, and so $h = h' + \gamma$. 

At what learning rate is the condition violated? To leading orders in $\lambda$, this can be identified by expanding the logarithmic term up to the second order in $\lambda$:
\begin{equation*}
    \E_x[\log |1-\lambda h(x)|] =  - \lambda \E[h(x)] - \frac{1}{2}\lambda^2\E_x[h(x)^2] + {O}(\lambda^3).
\end{equation*}
Ignoring the second-order term, we see that the dynamics always follow the sign of $\E[h(x)]$, in agreement with the belief that the GD algorithm always escapes a strict saddle. When the second-order term is taken into consideration, the fluctuation of $h(x)$ now decides the attractivity of the critical point, which is attractive if 
\begin{equation}
    \lambda > 2\frac{-\E[h(x)]}{\E[h(x)^2]}.
    \label{eq: lr condition}
\end{equation}
This condition directly points to the reason why Type-II saddle points are difficult (or impossible) to escape: the critical point can be attractive even if $h<0$, when the point has become a saddle point. 
The r.h.s. of the condition also has a natural interpretation as a signal-to-noise ratio (SNR) in the gradient. The {numerator} is the Hessian of the original loss function, which determines the signal in the gradient. The denominator is the strength of the gradient noise in the minibatch \cite{wu2018sgd}. An illustration of this solution is given in Figure~\ref{fig:first phase diagram}. We show the probabilistic stability conditions for a rank-$1$ saddle point with a rescaling symmetry (see Section~\ref{sec: phases of learning sgd}). The loss function is $\ell(u,w)= -xy uw + o(u^2 + w^2)$. Here, the data points $xy=1$, and $xy=a$ are sampled with equal probability. These saddles appear naturally in matrix factorization problems and also in recent sparse learning algorithms \cite{poon2021smooth, poon2022smooth, ziyin2023sparsity, kolb2023smoothing}.

\subsection{Insufficiency of Norm-Stability}
Theorem~\ref{theo: 1d symmetry lyapunov condition} provides a perfect example to compare the probabilistic stability with the norm-stability. First of all, it is easy to see that if SGD converges to a point in $L_p$-norm, it must converge in probability. Therefore, moment stability always implies probabilistic stability, but not vice versa. Thus, norm stability is a more restricted notion than probabilistic stability. 

The following result shows that saddle points are always unstable for moment stability, which is constructively established by the following proposition.
\begin{theorem}\label{prop: moment insufficiency}
    Let $\theta_t$ follow Eq.~\eqref{eq: sgd dynamics linearized} around a critical point $\theta^*$ of an arbitrary loss function. Then, for any fixed $\lambda$,
    \begin{enumerate}[noitemsep,topsep=1pt, parsep=1pt,partopsep=1pt]
        \item there exists a data distribution such that $\theta_t$ is probabilistically stable but {not} $L_p$-stable;
        \item if $\theta^*$ is a saddle point and $p \geq 1$, the set of $\theta_0$ that is $L_p$-stable has Lebesgue measure zero.
    \end{enumerate}
\end{theorem}
Therefore, the $L_p$-stability is not useful in understanding the stability of SGD close to saddle points. One reason is that the outliers strongly influence the $L_p$ norm in the data, whereas the probabilistic stability is robust against such outliers, a point we will discuss in the next section. This theorem points to a very special but interesting property of Type-II saddles. Namely, that studying Type-II saddles really requires some rather special notion of attractivity. While probabilistic stability is likely not the only way to study Type-II saddles, it is the most well-understood one given the prior literature on the Lyapunov exponents of random matrix products.

\subsection{Lyapunov Exponent and Probabilistic Stability}
In reality, the notion of probabilistic stability is rather vague and cannot be directly experimentally explored. This problem can be avoided by noticing that the dynamics in Eq.~\eqref{eq: type two saddle dynamics} is the same as a random matrix product:
\begin{equation}
    \theta_t - \theta^* = \prod_i^d Z_t (\theta_0- \theta^*),
\end{equation}
where $Z_t = I - \lambda \hat{H}_t$. Thanks to the celebrated Furstenberg-Kesten theorem~\cite{furstenberg1960products}, one can prove that the condition for the probabilistic attractivity of the saddle point under SGD is the same as the Lyapunov exponent of the process being negative. The \textit{maximal Lyapunov exponent} of a point $\theta^*$ is defined as 
\begin{equation}
    \Lambda = \max_{\theta_0} \lim_{t\to \infty} {\frac{1}{t}} \E \left[\log \frac{\|\theta_t - \theta^*\|} {\|\theta_0\|} \right].
\end{equation}
Here, the expectation is taken over the random samplings of the SGD algorithm. In general, $\Lambda$ does not vanish. The following theorem shows that SGD is probabilistically stable at a point if and only if its Lyapunov exponent is negative. The maximum Lyapunov exponent is initialization-independent. If $\bar{H}$ is a $d$-by-$d$ matrix, a well-known fact is that the initialization-dependent Lyapunov exponent takes at most $d$ distinctive values. It is also interesting to study the Lyapunov exponent in each of these $d$ subspaces.

\begin{theorem}\label{theo: main theorem}
    Assuming that $\Lambda \neq 0$, the linearized dynamics of SGD is probabilistically stable at $\theta^*$ for any $\theta_0$ if and only if $\Lambda < 0$.
\end{theorem}

This result can also be proved by adapting the classical results in Ref.~\cite{bougerol1992strict}. Note that for a finite dataset size, one can always choose a learning rate such that $I-\lambda\hat{H}$ is all positive with probability $1$, this further makes it possible to apply the central limit theorem of the classical Furstenberg-Kesten theory.

For a finite-size dataset, it is straightforward to find an upper and lower bound for $\Lambda$. we can define $r_{\max}$ to be larger than the absolute value of the eigenvalues of $I-\lambda \hat{H}(x)$ for all $x$ (which exists because there is only finitely many $x$). Similarly, we can define $r_{\min}>0$ to be smaller than all the absolute values of all the eigenvalues of $I-\lambda \hat{H}(x)$ for all $x$. Therefore,
\begin{equation}
    \log r_{\min}<  \Lambda < \log r_{\max}.
\end{equation}
While it is in general challenging but worthwhile to theoretically estimate the exponent \cite{crisanti2012products, pollicott2010maximal, jurga2019effective}, it is a quantity that can be easily experimentally studied, and studying how Lyapunov exponents change with respect to hyperparameters of training algorithms.

Now, we give two quantitative estimates about when the Lyapunov exponent will be negative. This discussion also implies a sufficient but weak condition for a general type of multidimensional dynamics to converge in probability. Let $h^*(x)$ be the largest eigenvalue of $\hat{H}(x)$ and assume that $1-h^*(x)>0$ for all $x$. Then, the following condition implies that ${\theta}\to_p 0$: $\E_{x}[\log|1-\lambda h^*(x)|]<0$, which mimics the condition we found for rank-$1$ systems. An alternative estimate can be made by approximating every $\hat{H}(x)$ as a diagonal matrix, which is a common approximation in the literature that requires computing the Hessian or Fisher information of the neural networks. If $\hat{H}$ has rank $d$, {t}his reduces the problem to $d$ separated rank-1 dynamics, and Theorem~\ref{theo: 1d symmetry lyapunov condition} again gives the exact solutions in each subspace. Numerical evidence shows that the diagonal approximation quite accurately predicts the onset of low-rank behavior the actual rank (see Section~\ref{sec: phases of learning sgd} and Appendix~\ref{app: critical lr}). 

Another potential concern is whether this theorem is trivial for SGD at a high dimension in the sense that it could be the case that $\Lambda$ could be identically zero independent of the dataset. One can show that the Lyapunov exponent is generally nonzero for all datasets that satisfy a mild condition. Let $\E[\hat{H}]$ be full rank. By definition, 
\begin{align*}
    \Lambda &= \lim_{t\to\infty}\frac{1}{t}\E\left[\log \weight_0^T FF^T \weight_0 \right] = -\frac{2\lambda \weight_0^T  H \weight_0}{\|\weight_0\|^2} + O (\lambda^2),
\end{align*}
where $F= \prod_j^t (I - \lambda\hat{H}_{i_j})$.

Therefore, as long as $\lambda$ is sufficiently small, the sign of the Lyapunov exponent is opposite to the sign of the eigenvalues of $H$. This proves something quite general for SGD at an interpolation minimum: with a small learning rate, the model converges to the minimum exponentially fast, in agreement with common analysis in the optimization literature. See Figure~\ref{fig:first phase diagram}-right for the numerical computation of Lyapunov exponents of a matrix factorization problem and the corresponding phases.

\section{Application and Experiment}
\label{sec: exp}

\subsection{Solving lasso with SGD}

As a simple illustrative example, we first apply our theory to \textit{spred}, a recent algorithm for solving the lasso problem with gradient descent \cite{ziyin2023sparsity}. Here, the target function is $\ell =  ((u \odot w)x - y)^2$, where $\odot$ is the element-wise product, $u,\ w,\ ,x\in \mathbb{R}^{200}$, and $y\in\mathbb{R}$. It has been proved that all local minima of this loss function is identitcal to the unique solution of lasso, and thus have the same level of sparsity. We choose the dataset such that the lasso solution has $0.5$ sparsity. The solutions sparser than $50\%$ are the saddle points where $u_i =w_i = 0$ for some $i$. Apparently, each of these points is a Type-II saddle because the gradient of $w_i$ is proportional to $u_i$ and vice versa. The theory suggests that SGD is attracted to these saddle points at a large learning rate. See Figure~\ref{fig: sgd spred}.

\begin{wrapfigure}{r}{0.35\linewidth}
    \includegraphics[width=\linewidth]{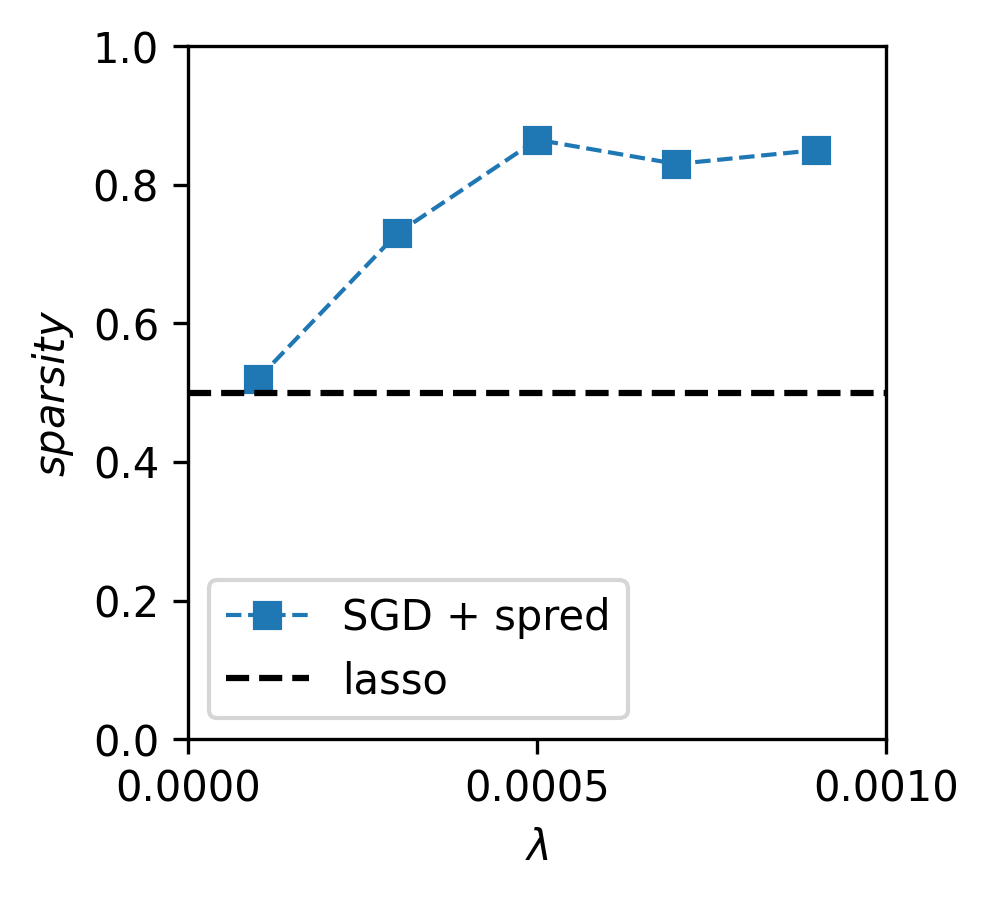}
    \caption{\small At a small learningrate, the \textit{spred} algorithm \cite{ziyin2023sparsity} converges to the ground truth solution of lasso. At a large learning rate, however, it is biased towards sparser solutions than the ground truth because the sparser solutions are Type-II saddles and thus attractive for SGD at a large learning rate.}

    \label{fig: sgd spred}
\end{wrapfigure}
\subsection{Comparison with moment stability}

Conventionally, it is thought that SGD can only select a local or global minimum as a solution. Our result suggests that saddle points can also be selected as solutions. In this section, we present one such case study that directly compares the attractivity of fixed points predicted by the probabilistic stability and the $L_p$ stability. We consider a two-layer network with a single hidden neuron with the swish activation function: $f(w,u,x) = u \times {\rm swish}(wx)$, where ${\rm swish}(x) = x \times {\rm sigmoid}(x)$. We generate $100$ data points $(x,y)$ as $y = 0.1 {\rm swish}(x) + 0.9\epsilon$, where both $x$ and $\epsilon$ are sampled from normal distributions. See Figure~\ref{fig:select minimum} for an illustration of the training loss landscape. There are two local minima: solution A at roughly $(-0.7, -0.2)$ and solution B at $(1.1, -0.3)$. Here, the solution with better generalization is A because it captures the correct correlation between $x$ and $y$ when $x$ is small. Solution A is also the sharper one; its largest Hessian eigenvalue is roughly $h_a = 7.7$. Solution B is the worse solution; it is also the flatter one, with the largest Hessian value being $h_b=3.0$. There is also a saddle point C at $(0,0)$, which performs significantly better than B and slightly worse than A in terms of generalization.

If we initialize the model at A, $L_2$ stability theory would predict that as we increase the learning rate, the solution moves from the sharper minimum A to the flatter minimum B when SGD loses $L_2$ stability in A; the model would then lose total stability once SGD becomes $L_2$-unstable at B. As shown by the red arrows in Figure~\ref{fig:select minimum}. In contrast, probabilistic stability predicts that SGD will move from A to C as C becomes attractive and then lose stability, as the black arrows indicate. See the right panel of the figure for the comparison with the experiment for the model's generalization performance. The dashed lines in the middle and right panels show the predictions of the $L_2$ stability and probabilistic theories, respectively. We see that the probabilistic theory predicts both the error and the place of transition right, whereas $L_2$ stability neither predicts the right transition nor the correct level of performance. If we initialize at B, the flatter minimum, $L_2$ stability theory would predict that the solution will only have one jump from B to divergence as we increase the learning rate. Thus, from $L_2$ stability, SGD would have roughly the performance of B until it diverges, and having a large learning rate will not help increase the performance. In sharp contrast, the probabilistic stability predicts that the solution will have two jumps: it stays at B for a small $\lambda$ and jumps to C as it becomes attractive at an intermediate learning rate. The model will ultimately diverge if C loses stability. Thus, our theory predicts that the model will first have a bad performance, then show a better performance at an intermediate learning rate, and finally diverge. See the middle panel of Figure~\ref{fig:select minimum}. We see that the prediction of the probabilistic stability agrees with the experiment and correctly explains why SGD leads to better performance.

\begin{figure*}
    \centering
    \includegraphics[width=\linewidth]{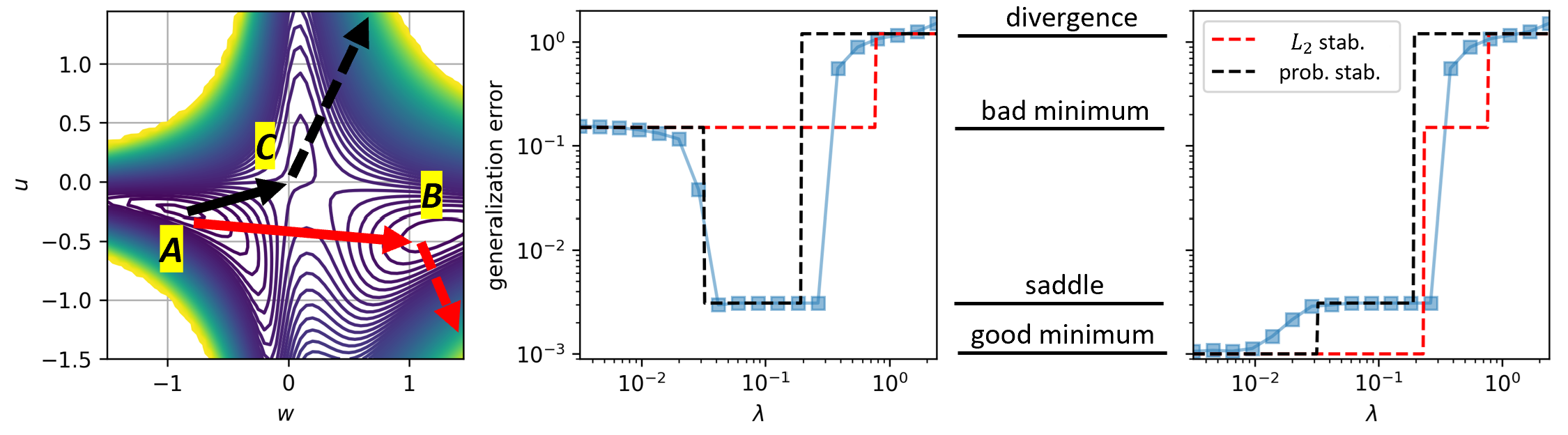}
    \caption{\small \textbf{How SGD selects a solution}. \textbf{Left}: The landscape of a two-layer network with the swish activation function \cite{ramachandran2017searching}. The black arrow corresponds to the experimental trajectory and the prediction of probabilistic stability, while the red arrow corresponds to the (false) prediction of the $L_2$ stability. \textbf{Middle, Right}: the generalization performance of the model for different learning rates. \textbf{Middle}: Initialized at solution B, SGD first jumps to C and then diverges. \textbf{Right}: Initialized at A, SGD also jumps to C and diverges. In both cases, the behavior of SGD agrees with the prediction of the probabilistic stability instead of the $L_2$ stability. Instead of jumping between local minima, SGD, at a large learning rate, transitions from minima to saddles.}
    \label{fig:select minimum}
\end{figure*}

\subsection{Phases of Learning}\label{sec: phases of learning sgd}

A good observation is that almost all neural networks are initialized close to a Type-II saddle. Let 
$f(x) = W^{(D)}\sigma(W^{(D-1)}\sigma(...\sigma(W^{(1)}x + b^{(1)})) + b^{(D-1)})$
be a generic 
neural network with 
depth $D$ and activation $\sigma$, and 
$W^{(D)}\in\mathbb{R}^{d_{D-1}\times d_y}$, $W^{(i)}\in\mathbb{R}^{d_{i-1}\times d_i}$, and $W^{(1)}\in\mathbb{R}^{d_{x}\times d_1}$ are of arbitrary dimensions that match the input and output dimensions. Here, $\sigma(x) = c_0 x + O(x^2)$ is any nonlinearity that is locally linear at $x=0$. Let $y=y(x)$ be the label of $x$. Assuming that the per-sample loss is differentiable, one can prove the following result, which is relevant for the standard small initialization of neural networks.
\begin{theorem}\label{theo: type II saddle at init}
    For any data distribution, loss function $\ell$, and any $\sigma$, the point $(u_i, w_i) = (0, 0)$ is a Type-II saddle of $\ell(f(x),y(x))$ if it is not a local minimum. 
\end{theorem}

The details of the proof can be found in section~\ref{ss:proof type II init}. Now, we solve an analytical model that exhibits complex and dramatic phase transition-like behaviors as we change the learning rate of SGD. We will see that the solutions we obtained from this model are relevant for understanding the initialization of real neural networks.

\begin{wrapfigure}{r}{0.38\linewidth}
    \includegraphics[width=\linewidth]{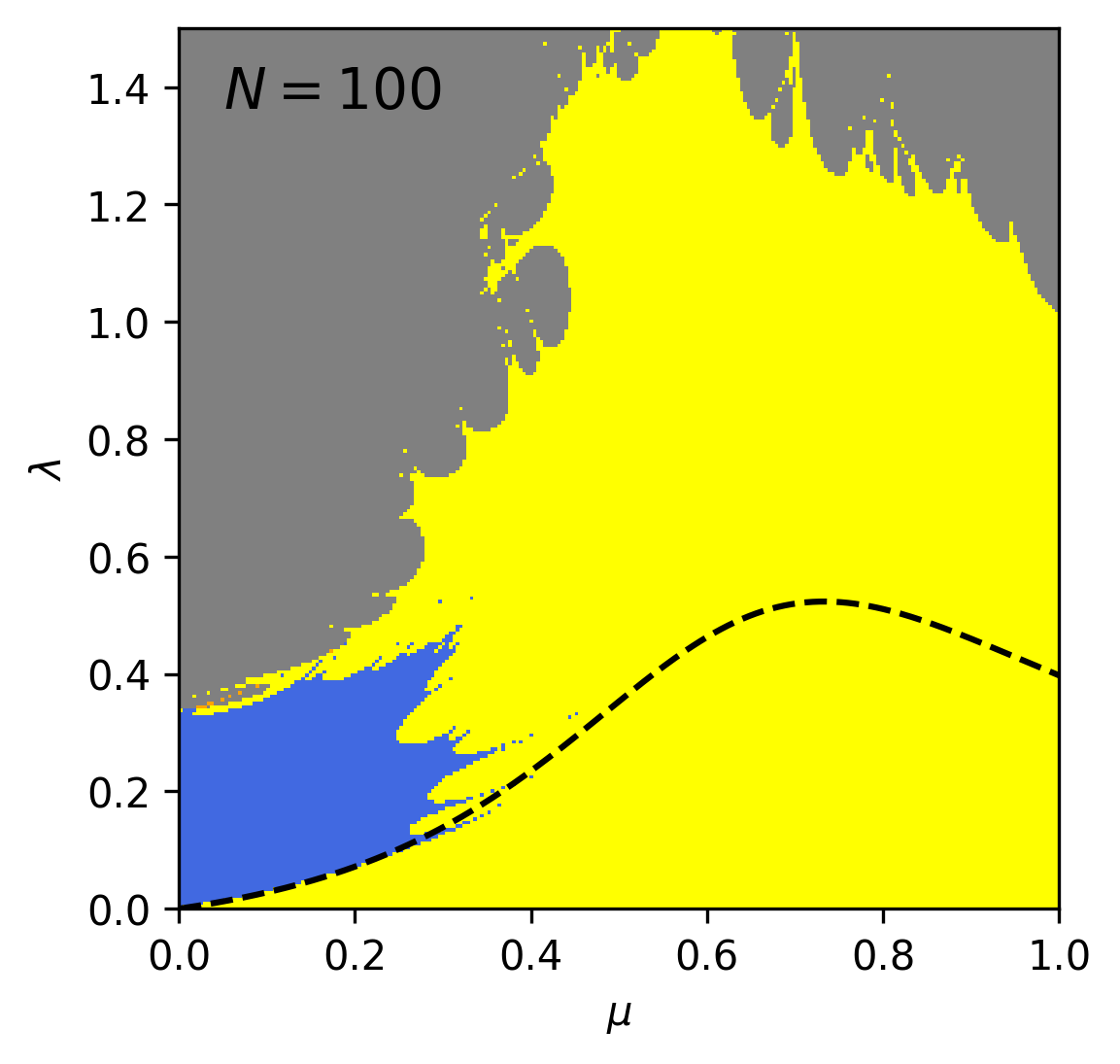}
    \caption{\small Phase diagram at $N=100$. The colors indicate the phases as in Figure~
    \ref{fig:first phase diagram}. At a finite-size, the phase boundaries have a fractal-like structure. The bottom-left of the phase diagram has a smooth boundary and is shared across almost all phase diagrams we plotted.}
    \label{fig: phase diagram 100 main}
\end{wrapfigure}

As an analytically tractable model, let the loss function be $\E_x[(\sum_i u^{(i)} \sigma(w^{(i)} x) - y)^2/2]$. We let $c_0=1$ and both $x$ and $y$ be one-dimensional. At $u,\ w\approx 0$, the model $u^Tw$ is either rank-$1$ or rank-$0$. The rank-$0$ point where $u^{(i)}=w^{(i)}=0$ for all $i$ is a saddle point as long as $\E[xy] \neq 0$. Consider the linearized dynamics around the saddle at $w^{(i)}=u^{(i)}=0$. The expanded loss function takes the following form: $\ell(u,w {; x, y}) = -{xy} \sum^d_i u^{(i)}w^{(i)} + const.$ For learning to happen, SGD needs to escape from the saddle point. Let us consider a simple data distribution where $xy=1$ and $xy=a$ with equal probability. When $a>-1$, \textit{correct} learning happens when ${\rm sign}(w)={\rm sign}(u)$. We thus focus on the case of $a>-1$. The case of $a<-1$ is symmetric to this case up to a rescaling. This example is already presented in Figure~\ref{fig:first phase diagram}. {There are five phases of learning in this simple example}

\begin{itemize}[leftmargin=10pt, noitemsep,topsep=0pt, parsep=1pt,partopsep=1pt]
    \item Ia. correct learning with prob. and norm stability ($w_t - u_t \to_{L_2} 0 $, $w_t + u_t$ diverges);
    \item Ib. correct learning with prob. but not norm stability ($w_t - u_t \to_{p} 0$, $w_t - u_t \not\to_{L_2} 0$,  $w_t + u_t$ diverges);
    \item II. incorrect learning under probabilistic stability ($w_t - u_t$ diverges, $w_t + u_t \to_p 0$);
    \item III. convergence to low-rank saddle point ($w_t - u_t \to_{p} 0$, $w_t + u_t \to_{p} 0$);
    \item IV. completely unstable ($w_t + u_t$, $w_t - u_t$ diverges in p.).
\end{itemize}

Two aspects agree well with the theory: (1) SGD can indeed converge to low-rank saddle points; however, this happens only when the gradient noise is sufficiently strong and when the learning rate is large (but not too large); (2) the region for convergence to saddles (region III) is exclusive with the region for convergence in mean square (Ia), and thus one can only understand the saddle-seeking behavior of SGD within the proposed probabilistic framework. {Let $B$ denote a mini-batch and $S$ be its cardinality. The following proposition precisely characterizes the phase boundaries.} 

\begin{proposition}\label{prop: phase diagram}
    For any $w_0, u_0 \in \mathbb{R}/\{0\}$. $w_t - u_t \to_p 0$ if and only if $\E_{B}[\log|1-\lambda \sum_{(x, y)\in B}xy {/S}|] <0$. $w_t + u_t$ converges to $0$ in probability if and only if $\E_{B}[\log|1+\lambda {\sum_{(x, y)\in B}} xy {/S}|] <0$.
\end{proposition}

This shows that the phase diagram of SGD strongly depends on the data distribution, and it is interesting to explore and compare a few different settings. Now, we consider a size-$N$ Gaussian dataset. Let $x_i \sim \mathcal{N}(0, 1)$ and noise $\epsilon_i \sim \mathcal{N}(0, 4)$, and generate a noisy label $y_i=\mu x_i + (1-\mu) \epsilon_i$. See the phase diagram for this dataset in Figure~\ref{fig: phase diagram 100 main}. We see that the phase diagram has a rich structure at a finite size. There are three rather surprising observations about the phase diagrams: (1) as $N\to \infty$, the phase diagram becomes smoother and smoother and each phase takes a connected region (cf. experiments in Appendix~\ref{app sec: finite size}); (2) phase II seems to disappear as $N$ becomes large; (3) the lower part of the phase diagram seems universal, taking the same shape for all samplings of the datasets and across different sizes of the dataset. This suggests that the convergence to low-rank structures can be a universal aspect of SGD dynamics, which corroborates the widely observed phenomenon of collapse in deep learning \cite{papyan2020prevalence, wang2022posterior, tian2022deep}. The theory also shows that if we fix the learning rate and noise level, increasing the batch size makes it more and more difficult to converge to the low-rank solution (see Figure~\ref{fig:resnet18 different batchsize}, for example). This is expected because the larger the batch size, the smaller the effective noise in the gradient.

\subsection{Initialization of Deep Neural Networks}

As Theorem~\ref{theo: type II saddle at init} shows, almost all neural networks are close to Type-II saddles at initialization. This means that they will also have significant problems in initiating training if the SNR in the gradient is small. To verify this, we construct and modify datasets to exhibit different SNR during training for deep nonlinear networks.

We start with a controlled experiment where, at every training step, we sample input $x \sim \mathcal{N}(0, I_{200})$ and noise $\epsilon \sim \mathcal{N}(0, 4I_{200})$, and generate a noisy label $y=\mu x + (1-\mu) \epsilon$. Note that $1-\mu$ controls the level of the noise. Training proceeds with SGD on the MSE loss. We train a two-layer model with the architecture: $200 \to 200 \to 200$. See Figure~\ref{fig: nn low rank} for the theoretical phase diagram, which is estimated under the diagonal approximation. Under SGD, the model escapes from the saddle with a finite variance to the right of the dashed line and has an infinite variance to its left. In the region $\lambda \in (0, 0.2)$, this loss of the $L_2$ stability condition coincides with the condition for the convergence to the saddle. The experiment shows that the theoretical boundary agrees well with the numerical results. {The Adam optimizer \cite{journals/corr/KingmaB14_adam} also have a similar phase diagram. See Appendix~\ref{app sec: exp}. This suggests that the effects we studied are rather universal, not just a special feature of SGD.}

Then, we train independently initialized ResNets on CIFAR-10 with SGD. The training proceeds with SGD without momentum at a fixed learning rate and batch size $S=32$ (unless specified otherwise) for $10^5$ iterations. Our implementation of Resnet18 contains $11$M parameters and achieves $94\%$ test accuracy under the standard training protocol, consistent with the established values. To probe the effect of noise, we artificially inject a dynamical label noise during every training step, where, at every step, a correct label is flipped to a random label with probability $noise$, and we note that the phase diagrams are similar regardless of whether the noise is dynamical or static (where the mislabelling is fixed). See Figure~\ref{app fig: cifar static} for the phase diagram of static label noise. Interestingly, the best generalization performance is achieved close to the phase boundary when the noise is strong. This is direct evidence that SGD noise has a strong regularization effect on the trained model. We also study the sparsity of the ResNets in different layers in Figure~\ref{app fig: resnet sparsity static}, and we observe that the phase diagrams are all qualitatively similar. Also, see Appendix~\ref{app sec: exp} for the experiment with a varying batch size. These experiments show that Type-II saddles are indeed a major obstacle in the initial phase of training. This result may also explain why the warmup is needed for training transformers. Training transformers is only successful if one starts with a small learning rate and increases it slowly toward a maximal value. If the training starts at a large learning rate, one often observes that the model becomes trapped in a low-capacity state soon after initialization. This agrees with the theoretical expectation that a small learning rate has a larger SNR for training.

\begin{figure*}[t!]
    \centering
    \includegraphics[width=0.40\linewidth]{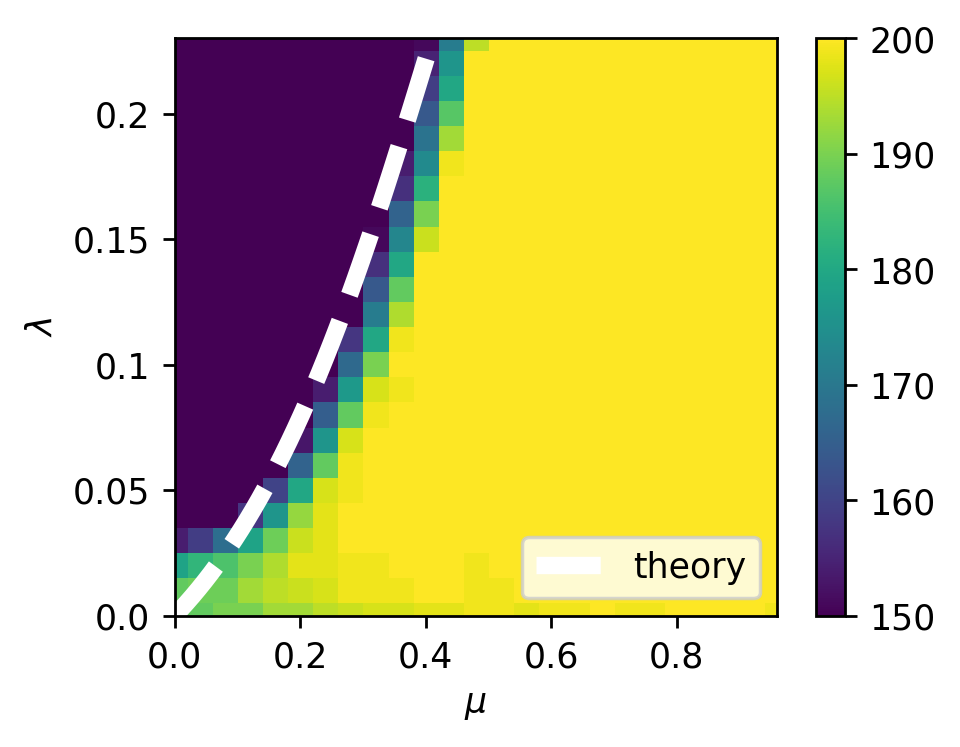}
    \includegraphics[width=0.40\linewidth]{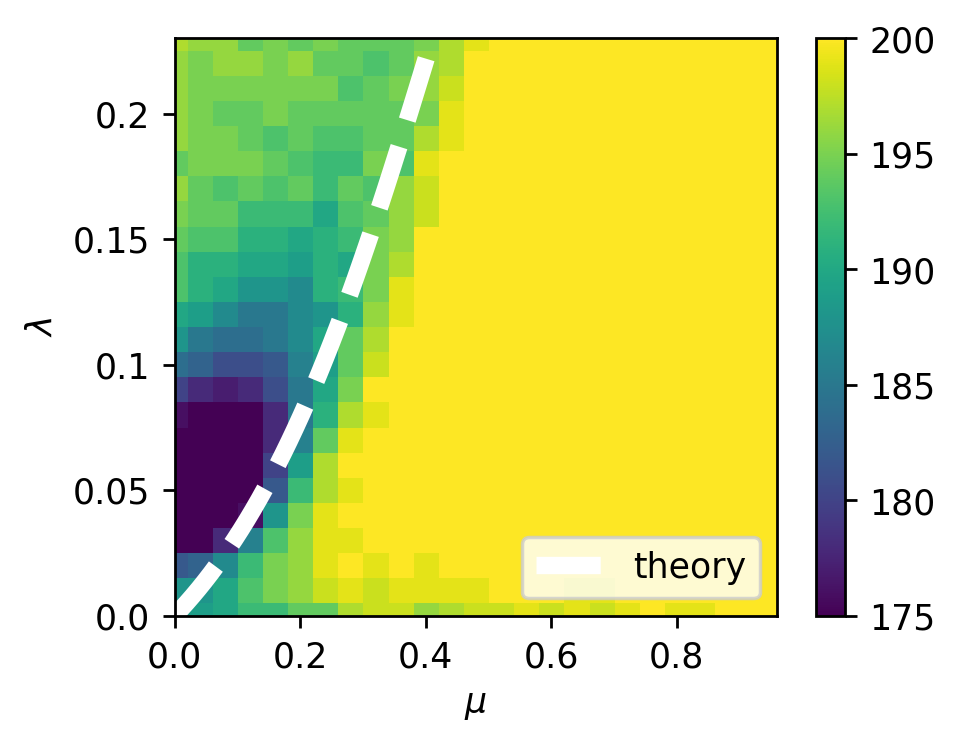}
    \caption{\small Convergence to low-rank solutions in nonlinear neural networks. At every training step, we sample input $x \sim \mathcal{N}(0, I_{200})$ and noise $\epsilon \sim \mathcal{N}(0, \sqrt{2}I_{200})$, and generate a noisy label $y=\mu x + (1-\mu) \epsilon$, where $1-\mu$ controls the level of the noise. We compute the rank of the second layer of the weight matrix after training. \textbf{Left}: Linear network. \textbf{Right}: tanh network. The white dashed line shows the theoretical prediction of the appearance of low-rank structure computed by numerically integrating the condition in Proposition~\ref{prop: phase diagram}.}
    \label{fig: nn low rank}

    \centering
    \includegraphics[width=0.40\linewidth]{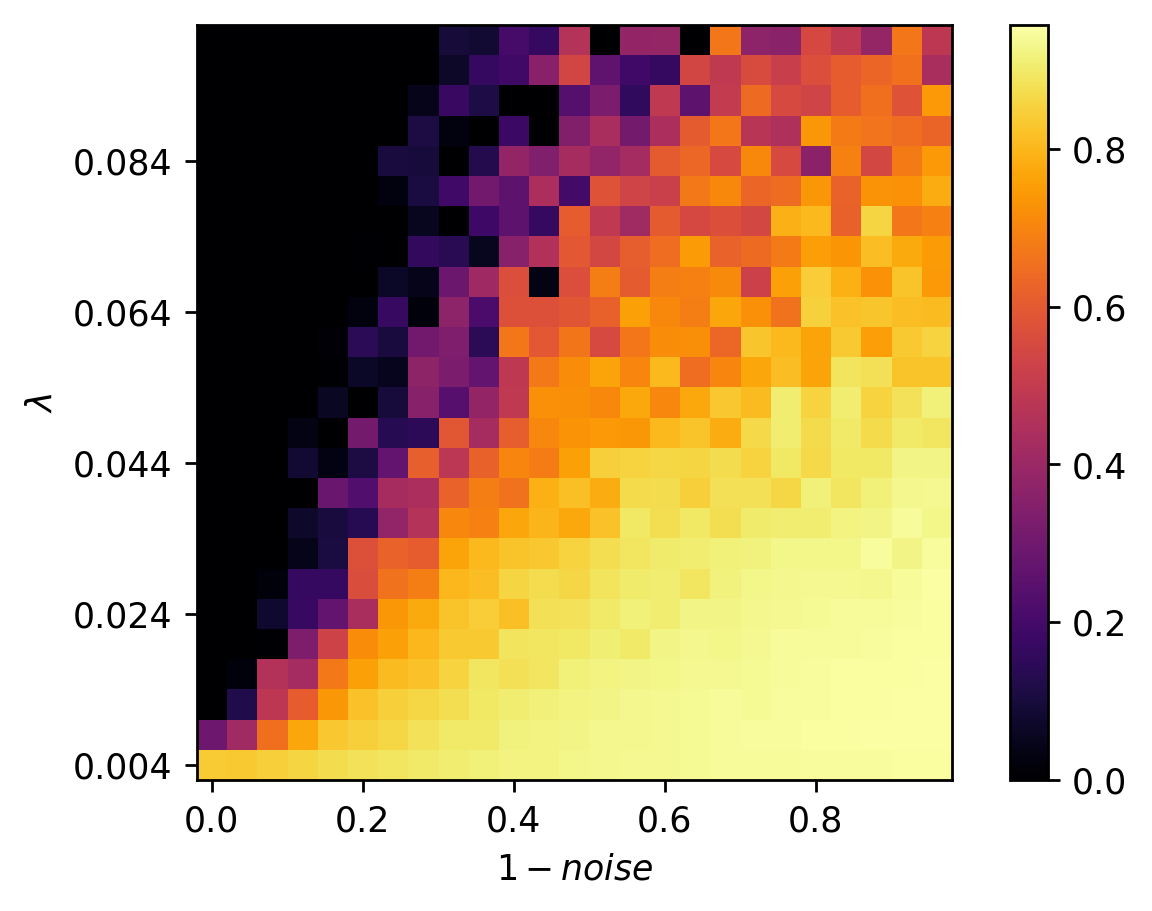}
    \includegraphics[width=0.40\linewidth]{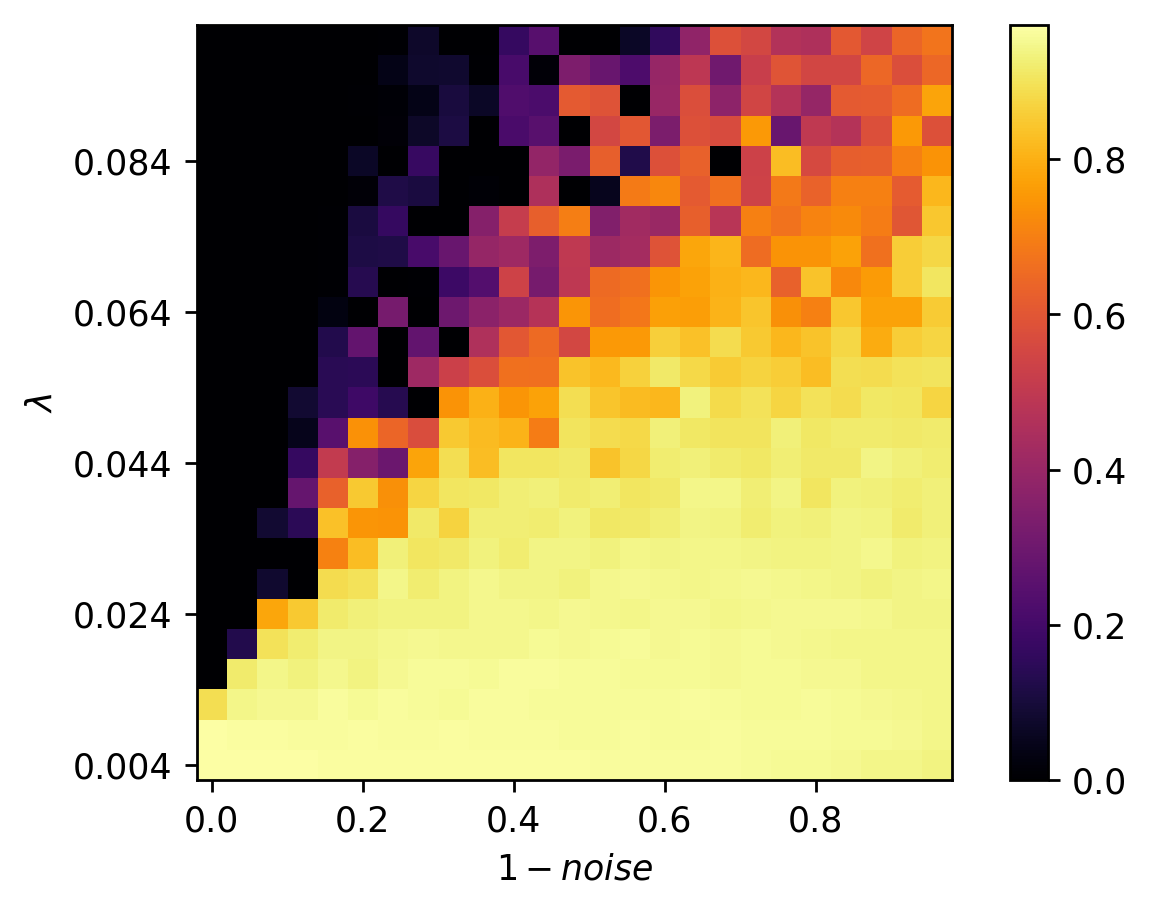}

    \caption{\small Density ($1-sparsity$) of the convolutional layers in a ResNet18, when there is static noise (mislabeling) in the training data. Here, we show the number of sparse parameters in the two of the largest convolutional layers, each containing roughly one million parameters in total. The figures respectively show layer1.1.conv2 (upper left), layer2.1.conv2 (upper right).}
    \label{app fig: resnet sparsity static}

\end{figure*}

\section{Discussion}

We have shown that special attention should be paid to the study of Type-II saddle points in neural networks. To this purpose, we demonstrate that the probabilistic stability serves as an essential notion for understanding the attractivity of Type-II saddle points as no moment-based notions of stability can be used to study when a saddle point becomes attractive. These effects are only present for SGD and not for GD, demonstrating that the algorithmic regularization due to SGD is qualitatively different from that of GD. We also link the saddle point problem to the conventional study of Lyapunov exponents in ergodic theory and dynamical systems, both of which can now be leveraged for us to design better algorithms of training. Our result also sheds light on how a solution in the loss landscape is chosen by SGD. In the probabilistic-stability perspective, SGD performs a selection between converging to saddles and to local minima, not between sharp minima and flat ones. A main limitation of our work is that we only considered simple analytical models on the theory side, but we believe they serve as good examples to connect contemporary deep learning research to the conventional fields of dynamical systems.

We have also suggested a new approach to the saddle point problem in neural networks. Depending on how one views the problem of Type-II saddles, there are two interesting future directions. If we regard the Type-II saddles as major obstacles to optimization, an important future problem is to design algorithms to better escape from them or to remove them from the loss landscape. Alternatively, the type-II saddles can function as effective capacity constraints on the models, and one might also be interested in leveraging them to better regularize neural networks.

\section*{Acknowledgment}

Liu Ziyin is supported by the JSPS fellowship during the writing of this paper. The research of B. Li is funded by the Chair \textit{Capital Markets Tomorrow: Modeling and Computational Issues} under the aegis of the Institut Europlace de Finance,  a joint initiative of Laboratoire de Probabilit\'es, Statistique et Modélisation (LPSM) / Université Paris Cit\'e and  Cr\'edit Agricole CIB, with the support of Labex FCD (ANR-11-LABX-0019-01). This work was supported by KAKENHI Grant No. JP22H01152 from the Japan Society for the Promotion of Science. We gratefully acknowledge the support from the CREST program ``Quantum Frontiers" (Grant No. JPMJCR23I1) by the Japan Science and Technology Agency.

\bibliography{ref}
\bibliographystyle{abbrv}

\newpage
\appendix
\onecolumn

\section{Experimental Concerns}\label{app sec: exp}

\subsection{Setting for Figure~\ref{fig: two types of saddles}}\label{app sec: figure 1 setting}
The following derivation gives the construction of the Type-I and Type-II saddle points in Figure~\ref{fig: two types of saddles}. Let us first study the case when there is a single neuron. The network is 
\begin{equation}
    f(x,\theta) = \sum_i u_i \sigma(w_i^Tx),
\end{equation}
where $\sigma={\rm ReLU}$ is the activation, and $x$ contains an element of $1$ to include the effect of having a bias. 

The loss function is $L = \frac{1}{N}\sum_x \ell(x)$, where
\begin{equation}
    \ell(x) = (f(x) -y(x))^2
\end{equation}
Without loss of generality, we can write $\E_x$ in the place of $\frac{1}{N}\sum_x$, where the expectation is taken over the training set. Let $\sigma ={\rm ReLU}$. We have that $\sigma''=0$. The gradients are 
\begin{equation}
    \nabla_{u} \ell(x) = \mathbbm{1}_> w^T x (uw^T x - y),
\end{equation}
\begin{equation}
    \nabla_{w} \ell(x) = \mathbbm{1}_> u  (uw^T x - y) x.
\end{equation}
where we have used $\mathbbm{1}_>$ as an indicator function of the event $w^Tx \geq 0$. The Hessian is 
\begin{equation}
    H(x) = \begin{bmatrix} \mathbbm{1}_>(w^Tx)^2   & \mathbbm{1}_> w^T x u x + \mathbbm{1}_> (u (w^T x) - y)x^T \\
     \mathbbm{1}_> w^T x u x +\mathbbm{1}_> (u (w^T x) - y)x & 
    \mathbbm{1}_> u^2  xx^T
    \end{bmatrix}
\end{equation}
There are two cases: (1) $u =0$ and (2) $u\neq 0$. 

For the case $u=0$, we have that $\nabla_w L= 0$. For $\nabla_u L =0$, we must have
\begin{equation}
    \nabla_{u} L = \E [\mathbbm{1}_> w^T x (- y)] = 0
\end{equation}
which is solved by any $w$ such that $w^T\E [\mathbbm{1}_> xy] =0$. These solutions are all saddle points. To see this, note that the sample Hessian is:
\begin{equation}
    H(x) =\begin{bmatrix}  \mathbbm{1}_>(w^Tx)^2   & -  \mathbbm{1}_>  yx^T \\
     -\mathbbm{1}_>  yx & 
    0
    \end{bmatrix},
\end{equation}

The expected Hessian is:
\begin{equation}
    \E[H(x)] =\begin{bmatrix} \E [\mathbbm{1}_>(w^Tx)^2]   & -\E  [\mathbbm{1}_>  yx^T] \\
     -\E[\mathbbm{1}_>  yx] & 
    0
    \end{bmatrix},
\end{equation}
which is a rank-2 matrix whose eigenproblem is solved by eigenvalues
\begin{equation}
        \lambda_\pm =   \frac{1}{2} \left(\E [\mathbbm{1}_>(w^Tx)^2] \pm \sqrt{\E [\mathbbm{1}_>(w^Tx)^2]^2 + \|\E  [\mathbbm{1}_>  yx^T]\|^2 } \right),
\end{equation}
and eigenvectors:
\begin{equation}
    v_\pm = (\lambda_\pm ,  - \E  [\mathbbm{1}_>  yx^T]).
\end{equation}
Obviously, $\lambda_- <0$, and these solutions are strict saddle points, except for the directions of $w$ that $\E  [\mathbbm{1}_>  yx^T] = 0$. 

Now, a crucial observation is that the saddle points given by this condition can be divided into two classes with dramatically different properties under a stochastic optimization algorithm: (I) $w\neq 0$ 
and (II) $w=0$. For the type I saddle, we have that the sample-wise gradient is
\begin{equation}
    \nabla_u \ell = \mathbbm{1}_> w^T x (uw^T x - y),
\end{equation}
which is a random variable that does not vanish in general. This random noise facilitates the escape from the saddle.

For the type II, we have that with probability $1$,
\begin{equation}
    \nabla_u \ell = 0,
\end{equation}
which signals a vanishing noise, the lack of which means that this type of saddle is difficult to escape. 
Alternatively, let $u = u_0 +\Delta u$ and $w = w_0 + \Delta w$ denote a small deviation from the saddle $(u_0, w_0)$. Then,
\begin{equation}
    \nabla_u \ell (x) = \begin{cases}
        \Theta(1) & \text{if $w_0\neq 0$};\\
        O(\|\Delta w\|) & \text{if $w_0 = 0$}.
    \end{cases}
\end{equation}  
\begin{equation}
    \nabla_w \ell (x) = O(\|\Delta u\|).
\end{equation}

\subsection{Experimental Detail for Figure~\ref{fig:first phase diagram}-right}\label{app sec: exp detail}
The experiment is performed for a two-dimensional system whose dynamics is specified in \eqref{eq: sgd dynamics linearized}. The expectation of the Hessian $\E[\hat{H}]$ is chosen to be $\text{diag}(0.1, -0.1)$, while the noise is generated via a normal $2\times 2$ random matrix $M_\text{noise}$. The noisy Hessian is obtained as
\begin{equation}
    \hat{H} = \E[\hat{H}] + M_\text{noise}+ M_\text{noise}^T,
\end{equation}
and one can verify that such $\hat{H}$ is symmetric and consistent with our choice of $\E[\hat{H}]$. The initial state is sampled from a unit circle. The dynamics stops at time step $t$, and the Lyapunov exponent is calculated as $\frac{1}{t}\log||\weight_t||$, if one of the three following conditions is satisfied: $||\weight||$ reaches the upper cutoff of $10^{100}$; $||\weight||$ reaches the lower cutoff of $10^{-140}$; the preset maximal number of steps of $5000$ is reached. For each learning rate, the Lyapunov is obtained as the average of the results collected in $800$ independent runs.

\subsection{Phases of Finite-Size Datasets}\label{app sec: finite size}

See Figure~\ref{fig: finite size dataset}. 

\begin{figure}[b!]
    \centering
    \includegraphics[width=0.33\linewidth]{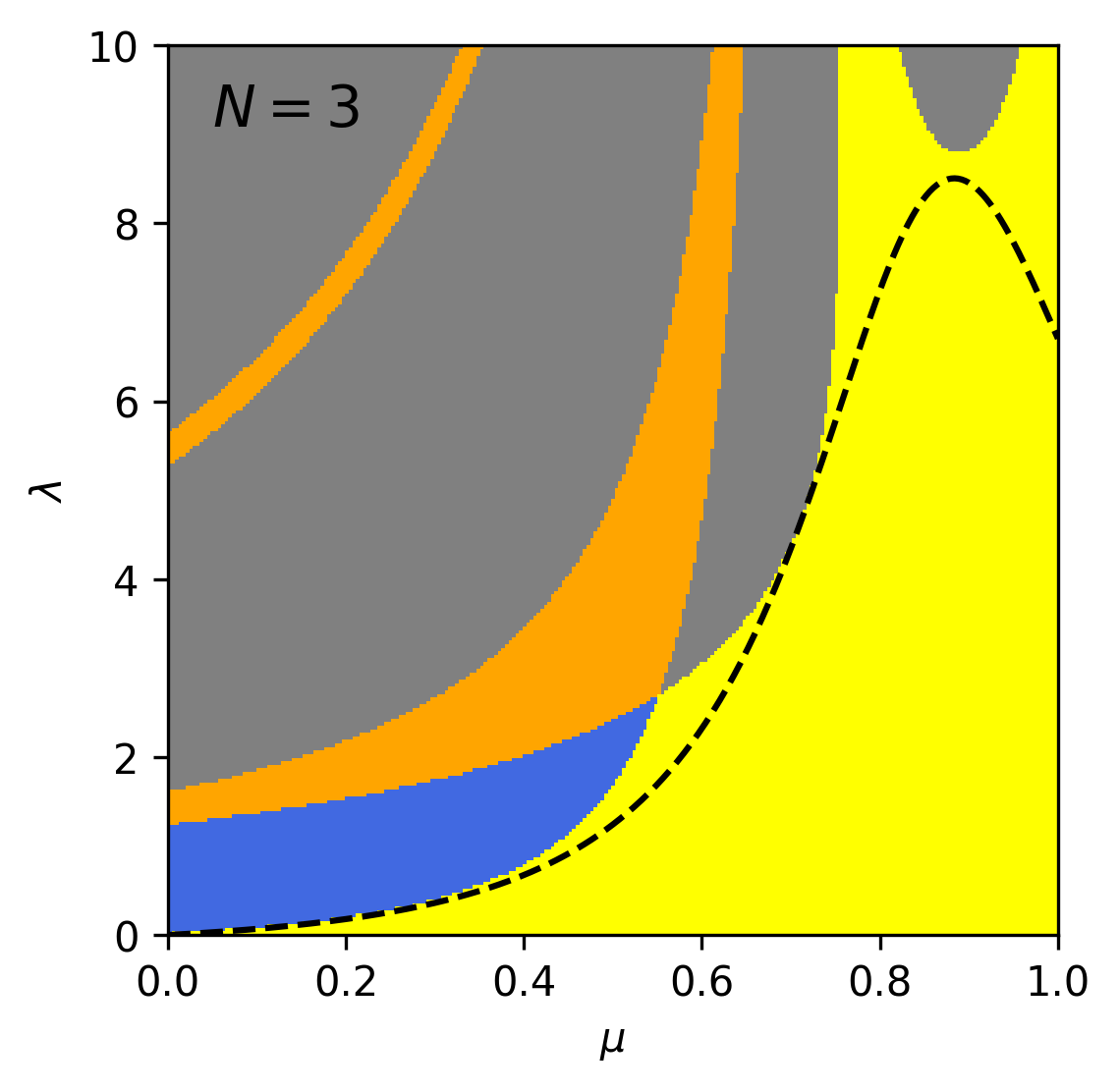}
    \includegraphics[width=0.33\linewidth]{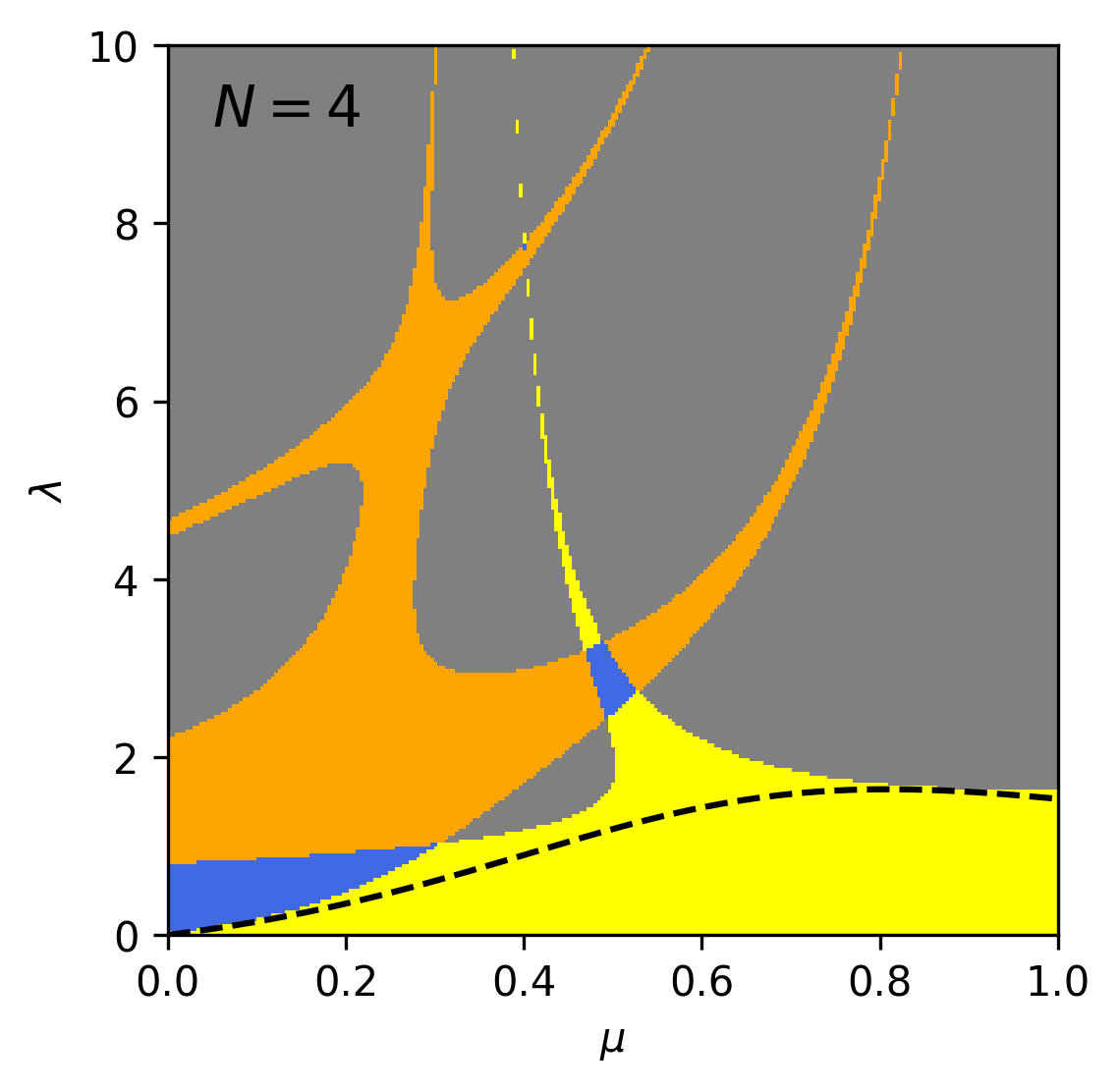}
    
    \includegraphics[width=0.33\linewidth]{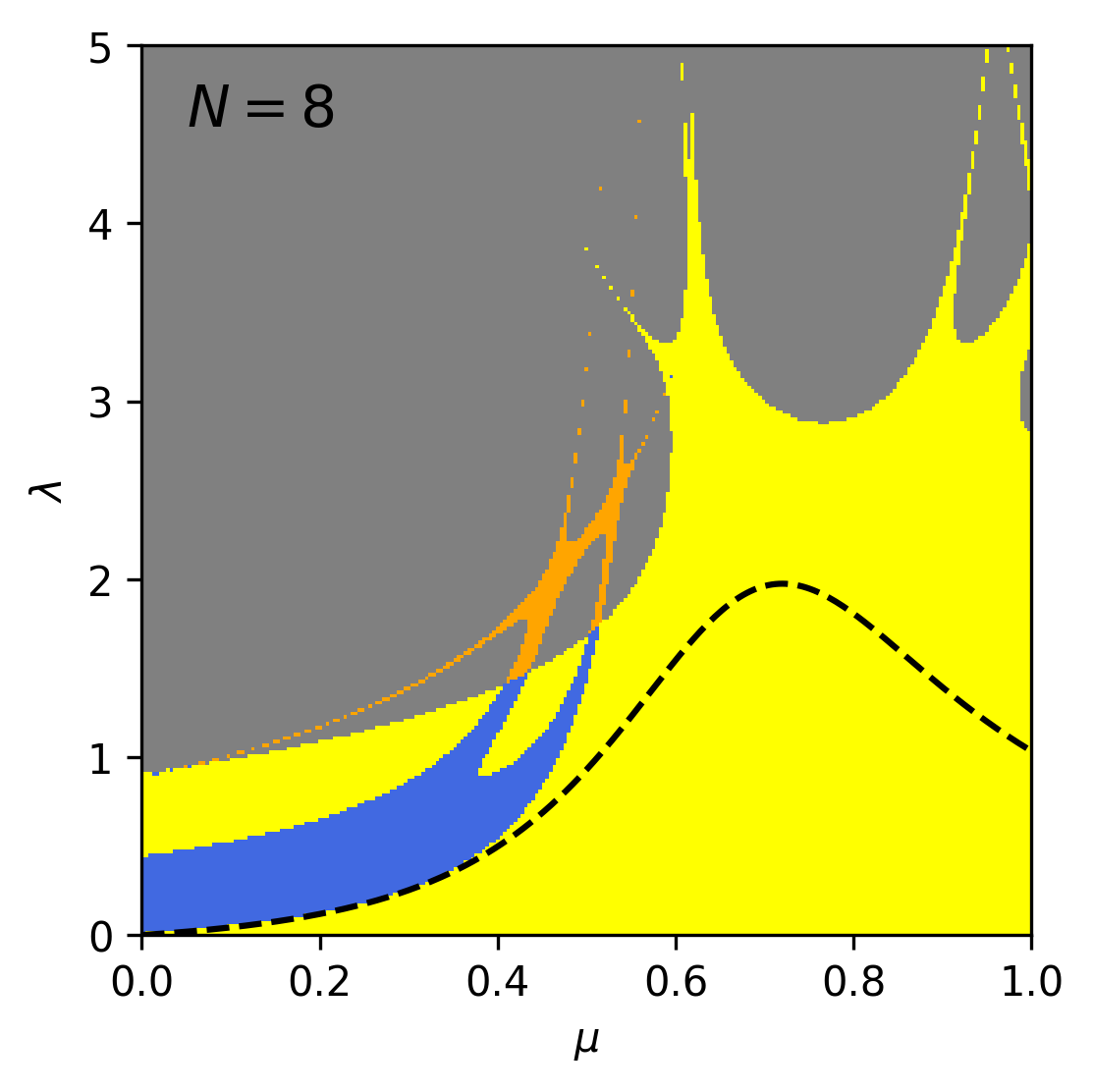}
    \includegraphics[width=0.33\linewidth]{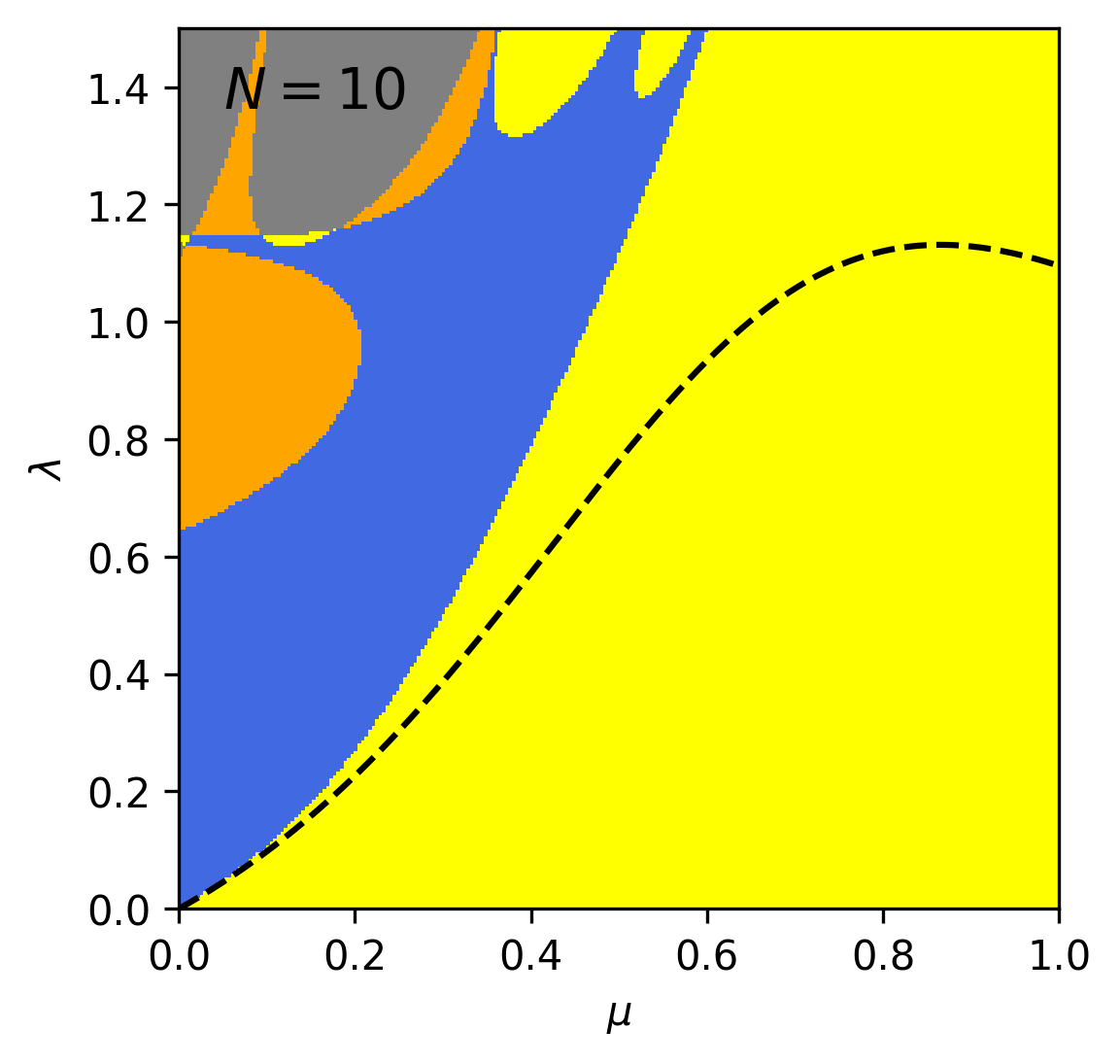}

    \includegraphics[width=0.33\linewidth]{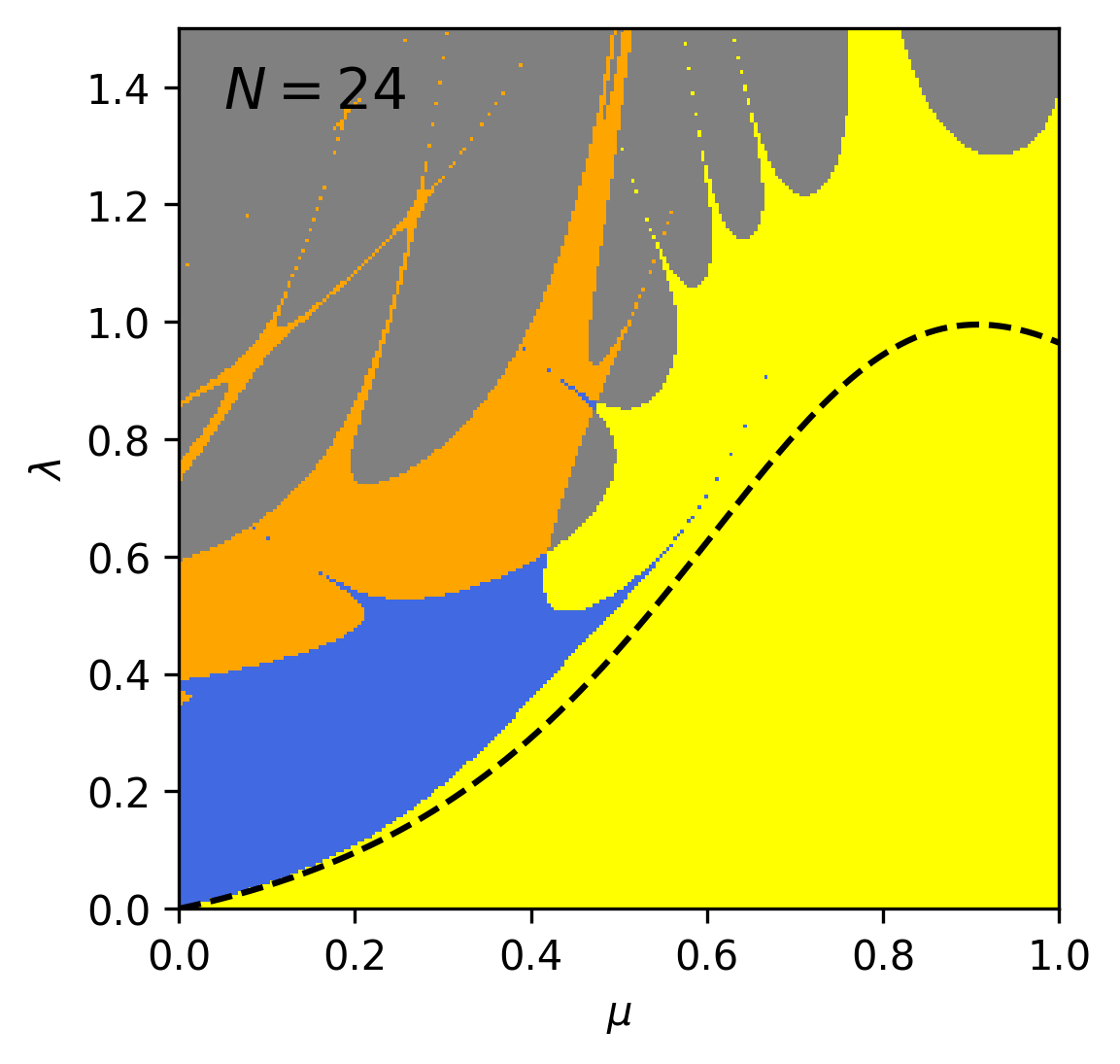}
    \includegraphics[width=0.33\linewidth]{plots/phase_diagram_100.png}

    \caption{\small \textbf{Phase diagrams of SGD stability for finite-size dataset}. The data sampling is the same as in Figure~\ref{fig: more phase diagrams}. From upper left to lower right: $N=3,\ 4,\ 8,\ 10,\ 24,\ 100$. As the dataset size tends to infinity, the phase diagram converges to that in Figure~\ref{fig: more phase diagrams}. The lower parts of all the phase diagrams look similar, suggesting a universal structure.}
    \label{fig: finite size dataset}
\end{figure}

\begin{figure*}[t!]
    \centering
    \includegraphics[width=0.32\linewidth]{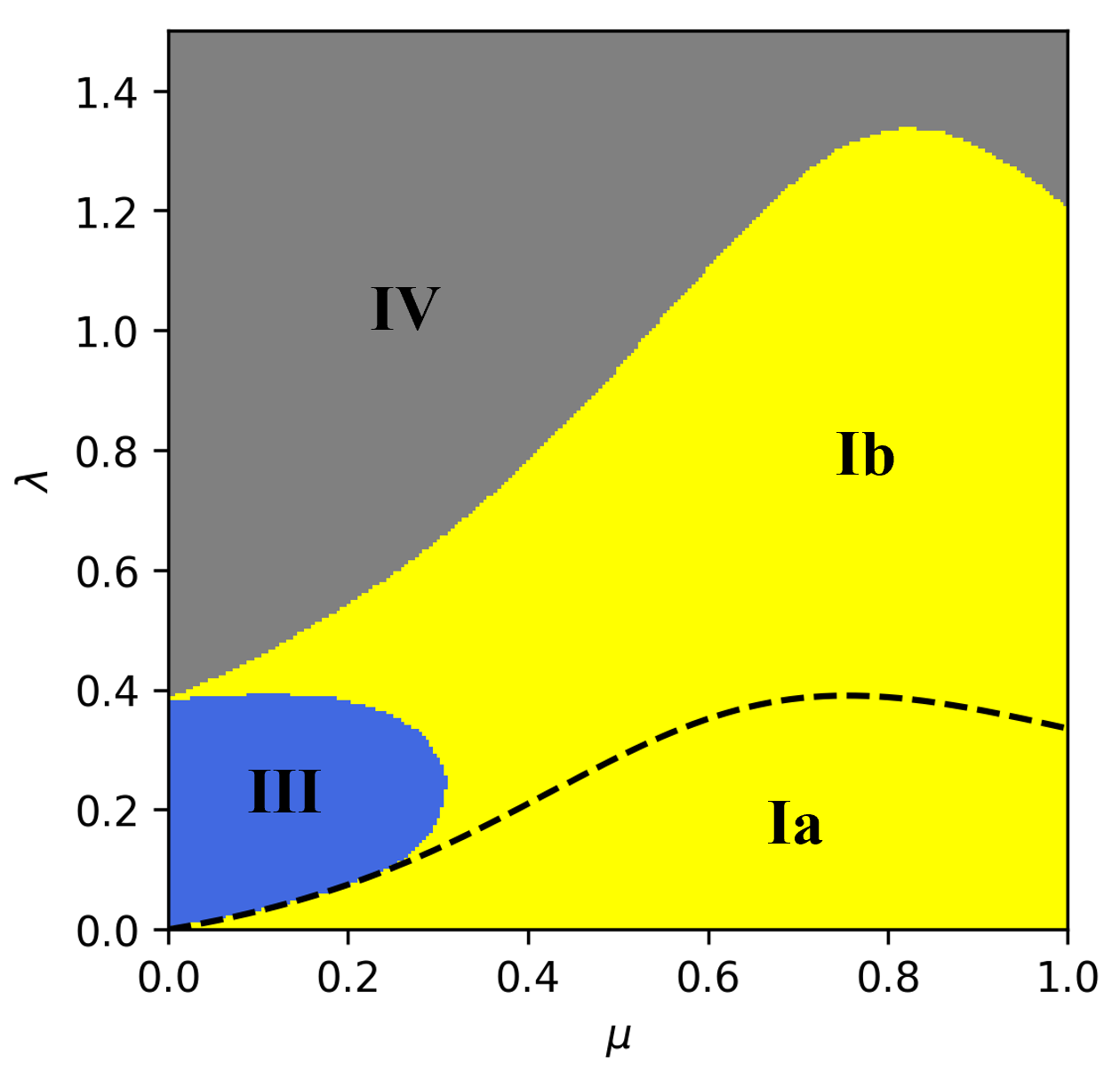}
    \includegraphics[trim={0 2mm 0 0}, clip, width=0.32\linewidth]{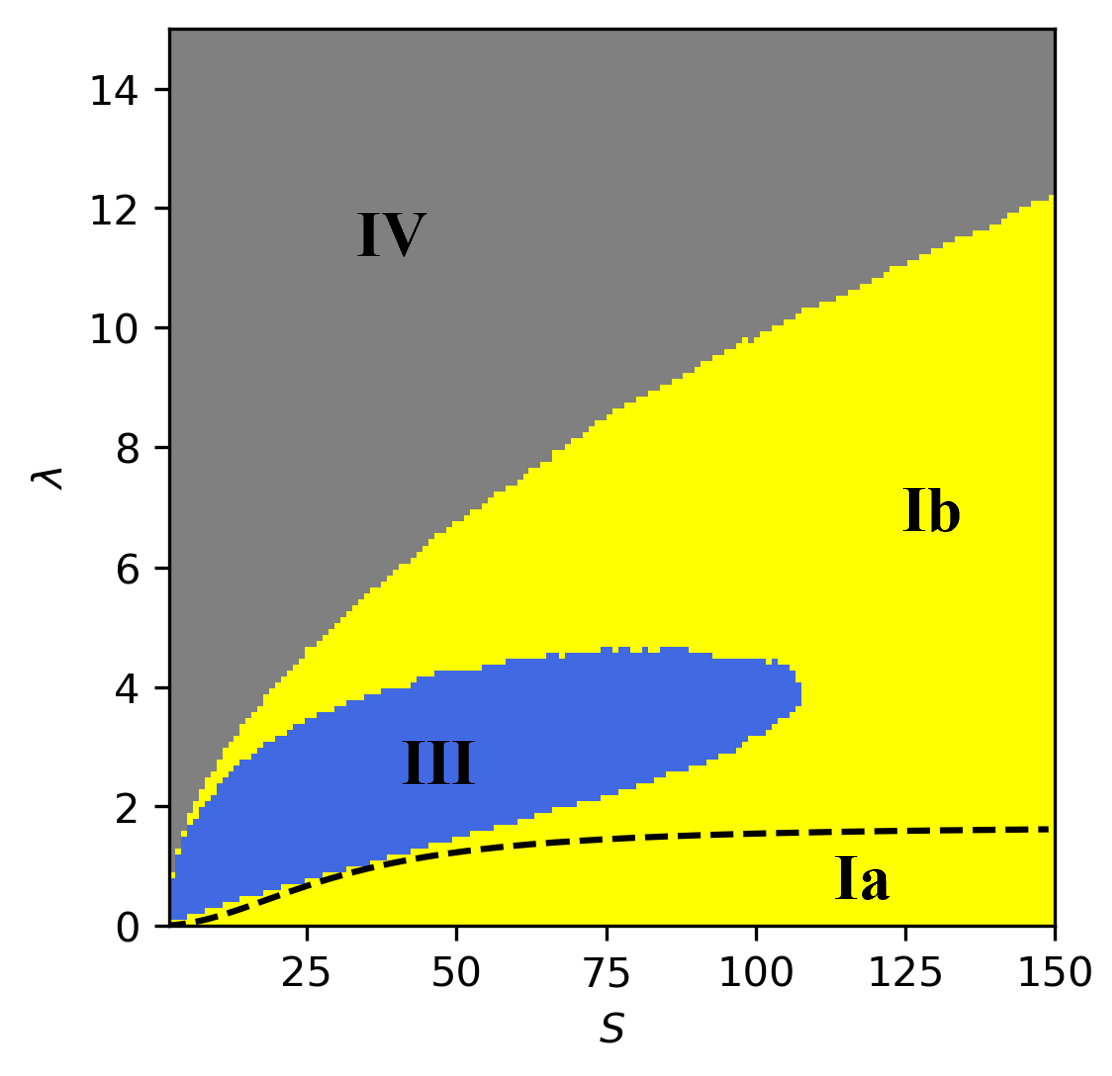}
    \caption{\small \textbf{Phase diagrams of SGD stability}. The definitions of the phases are the same as those in Figure~\ref{fig:first phase diagram}. We sample a dataset of size $N$ such that $x \sim \mathcal{N}(0, 1)$ and noise $\epsilon \sim \mathcal{N}(0, 4)$, and generate a noisy label $y=\mu x + (1-\mu) \epsilon$. Left: the $\lambda-\mu$ (noise level) phase diagram for $S=1$ and $N=\infty$. Right: The $\lambda-S$ (batch size) phase diagram for $\mu=0.06$ and $N=\infty$.}
    \label{fig: more phase diagrams}
\end{figure*}


\subsection{Resnet18 with Static Label noise}
See Figure~\ref{app fig: cifar static}, where the labels of the training set is corrupted at a fixed probability to a different fixed label during training. We see that the problem of initialization remains and features a similar phase boundary.

\begin{figure}
    \centering
    \includegraphics[width=0.35\linewidth]{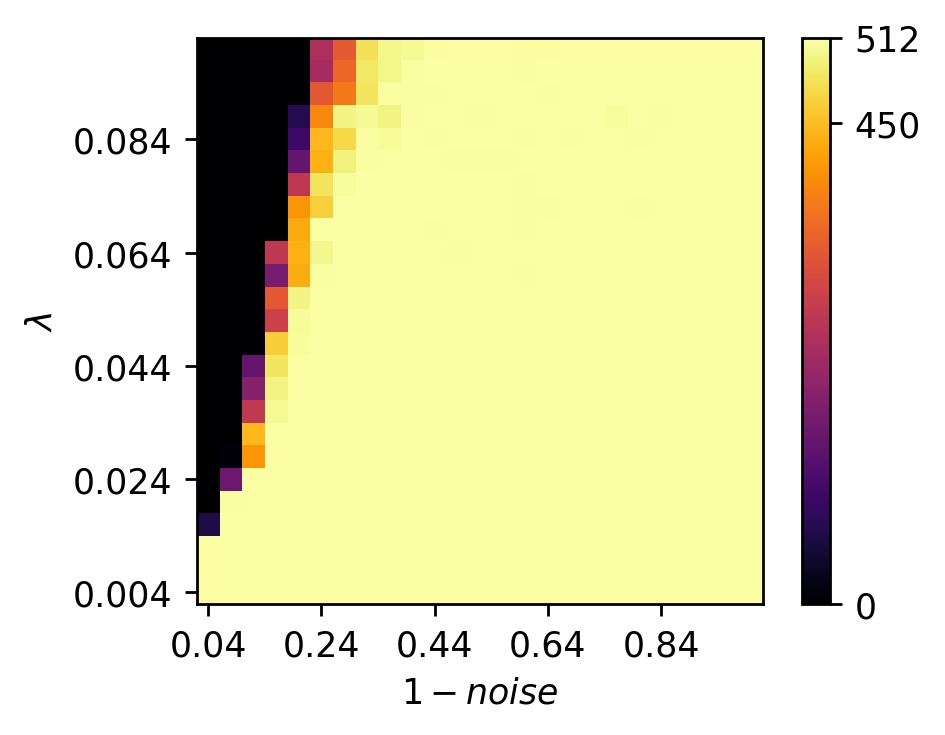}
    \includegraphics[width=0.35\linewidth]{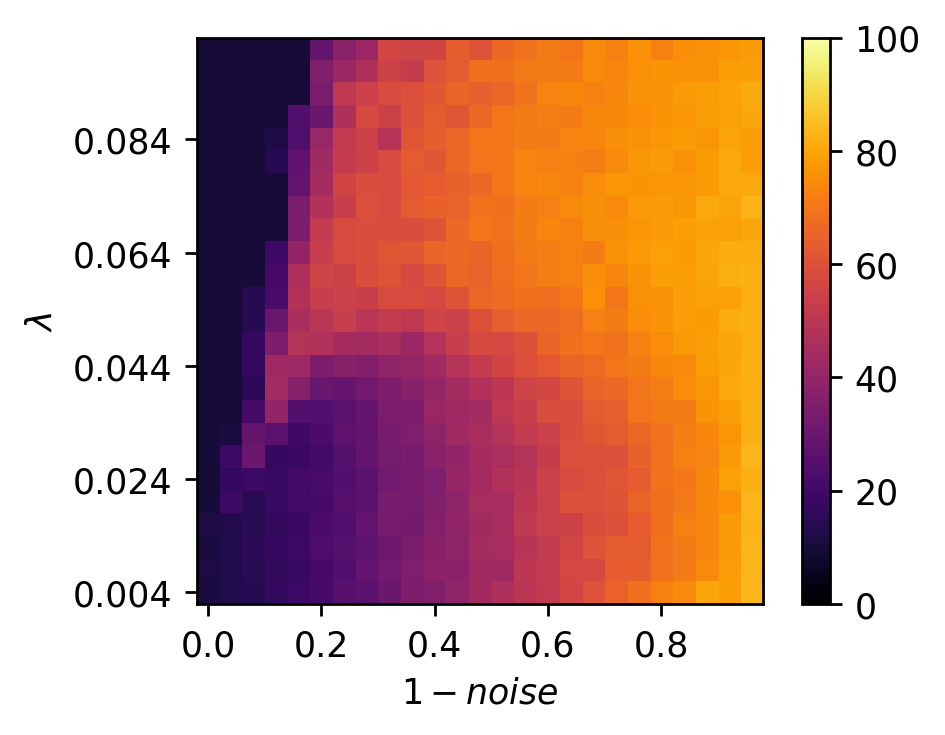}
    \caption{\small Rank (\textbf{left}) and test accuracy (\textbf{right}) of the ResNet18 trained in a data set with static noise. The transition of rank has a clear boundary. The model has a full rank but random-guess level performance for large noise and small learning rates. Here, \textit{noise} refers to the probability that the data point is mislabeled.}
    \label{app fig: cifar static}
\end{figure}

\subsection{Effect of Changing Batch size on Resnet18}
See Figure~\ref{fig:resnet18 different batchsize}.
\begin{figure}[t!]
    \centering
    \includegraphics[width=0.3\linewidth]{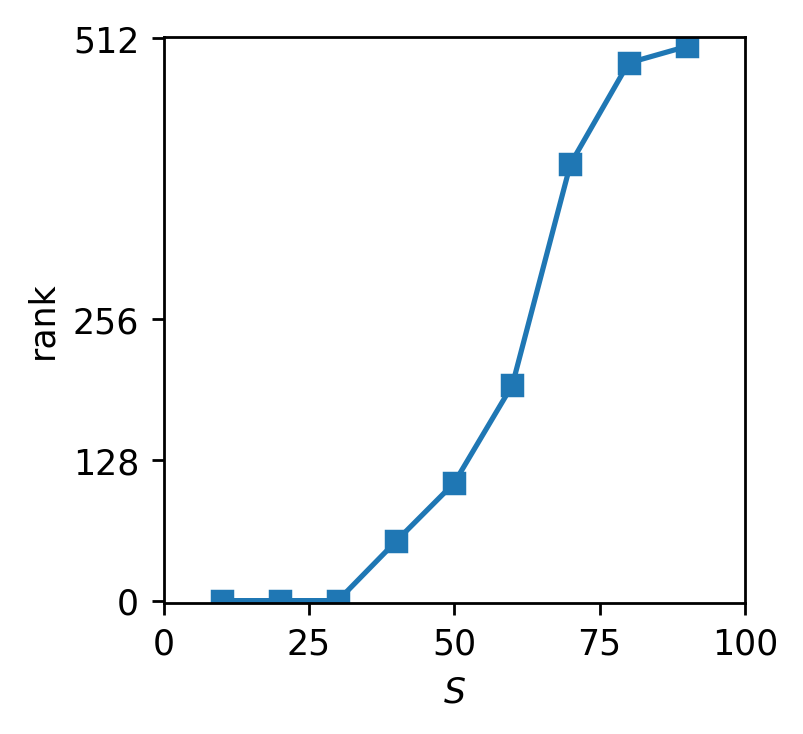}
    \caption{\small The rank of the penultimate-layer representation of Resnet18 trained with different levels of batch sizes. In agreement with the phase diagram, the model escapes from the low-rank saddle as one increases the batch size.}
    \label{fig:resnet18 different batchsize}
\end{figure}

\clearpage
\subsection{Diagonal Approximation}
\label{app: critical lr}
Here, we compare the empirical rank of the solution with the diagonally approximated critical learning rates obtained in \eqref{eq: lr condition}. See Figure~\ref{fig: crit lr}. The experiment is run on a two-layer fully connected linear network: $50 \to 50\to 50$, which is equivalent to a matrix factorization problem. The model is initialized with the standard Kaiming init. The dataset we consider is one with a sparse but full-rank signal. 

Let $\odot$ denote the Hadamard product. The input data is generated as $x= m \odot X$, where $X\sim \mathcal{N}(0, I_{50})$ and $m$ is a random mask where a random element is set to be $1$, and the rest is zero. The labels $Y$ are generated as $Y = \mu x + (1-\mu) (m \odot \epsilon)$, where $\epsilon \sim \mathcal{N}(0, 2 \text{diag}(0.01, 0.05, ...., 2.01))$ is the noise. 

\begin{figure}[t!]
    \centering
    \includegraphics[width=0.35\linewidth]{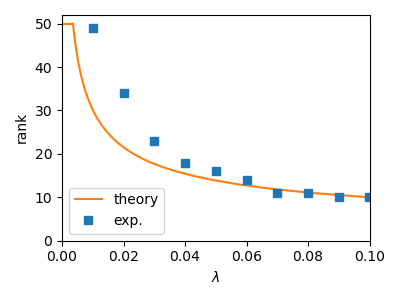}
    \includegraphics[width=0.35\linewidth]{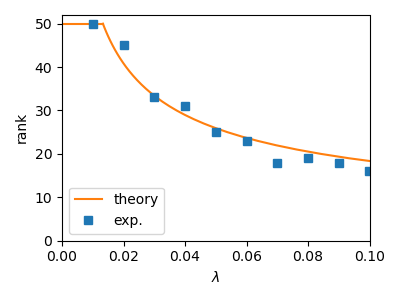}
    \includegraphics[width=0.35\linewidth]{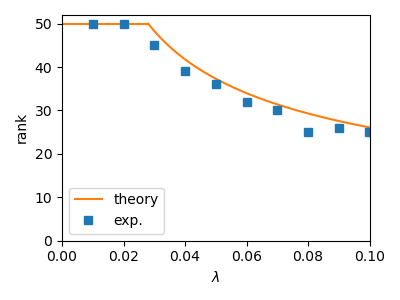}
    \includegraphics[width=0.35\linewidth]{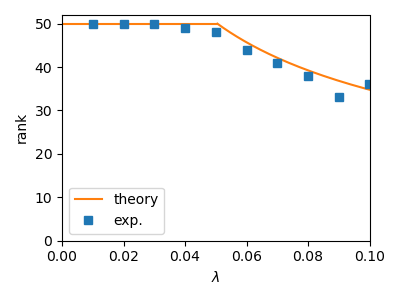}
    \caption{Rank vs learning rate in a vanilla matrix factorization problem for, from upper left to lower right, $\mu = 0.05,\ 0.15,\ 0.25,\ 0.35$. The theoretical curve is from the diagonal approximation where each subspace of the model collapses at the critical learning rate $\lambda = -2\frac{\E[h(x)]}{\E[h^2(x)]}$.}
    \label{fig: crit lr}
\end{figure}

{
\begin{figure}[t!]
    \centering
    \includegraphics[width=0.4\linewidth]{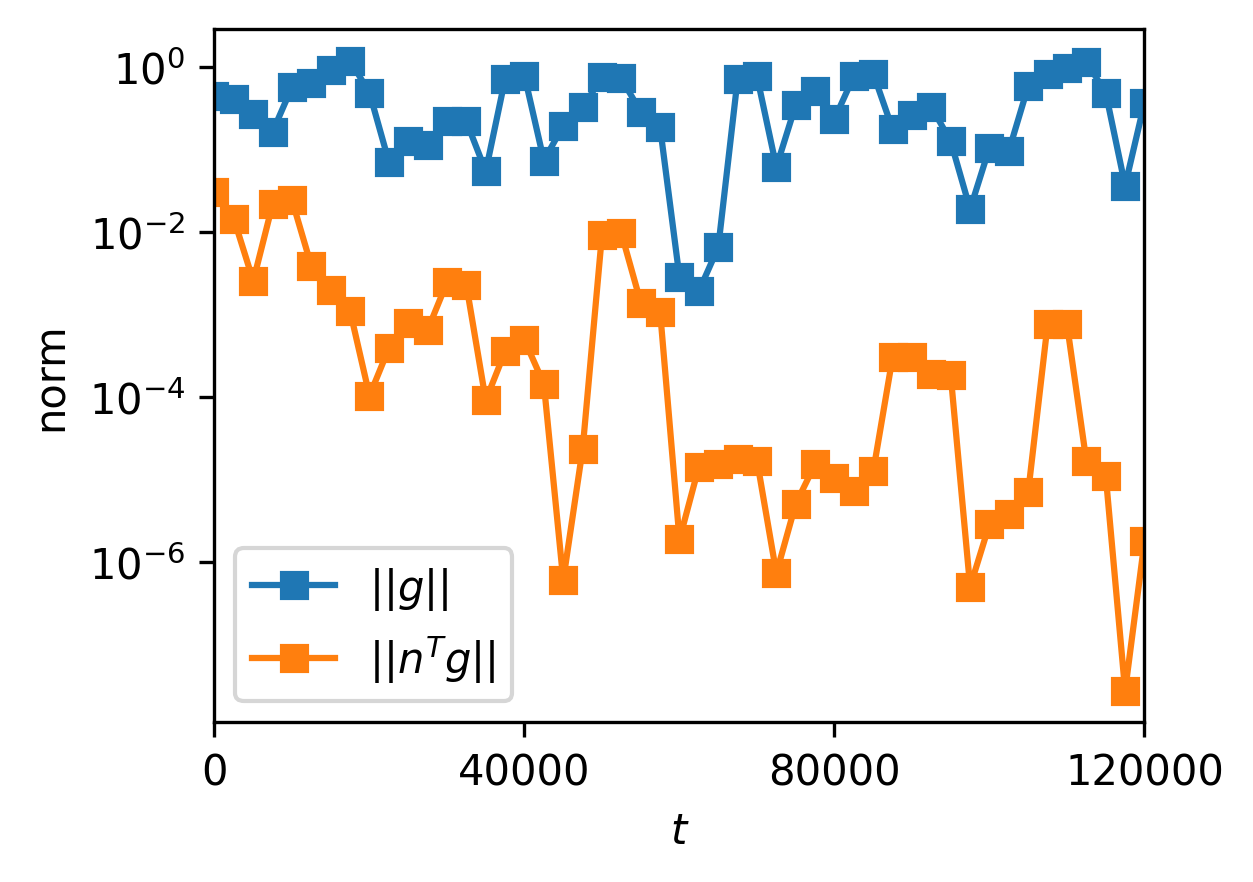}
    \caption{{Norms of gradient ($g$) of a matrix factorization problem trained with SGD. $n$ is a low-rank direction after training. During training, it is commonly the case the model does not converge to a stationary point but to a stationary distribution. Our theory is compatible with this case because it is possible and common for the model to converge to a point in some subspace, even if it is not converging to a point overall.}}
    \label{fig:stationary distribution}
\end{figure}

\subsection{Convergence to Stationary distributions}\label{app sec: stationary distribution}

As we mentioned in the main text, our theory is compatible with the case when the model converges to a stationary distribution but not a stationary point, which is more commonly the case during actual deep learning practice \cite{zhang2022neural}. The experimental setup is the same as in the previous section.

See Figure~\ref{fig:stationary distribution}, where we plot the norm of the gradient $g$ and the norm of $n^Tg$, where $n$ is a low-rank direction after training. Here, we see that the norm of the gradient does not converge to zero, but to a positive value, signaling a convergence to a stationary distribution. At the same time, the norm of $n^Tg$ does converge to zero, which means that in some subspace, the parameters do converge to a point. This justifies our argument.

}


\begin{figure}[t!]
    \centering
    \includegraphics[width=0.35\linewidth]{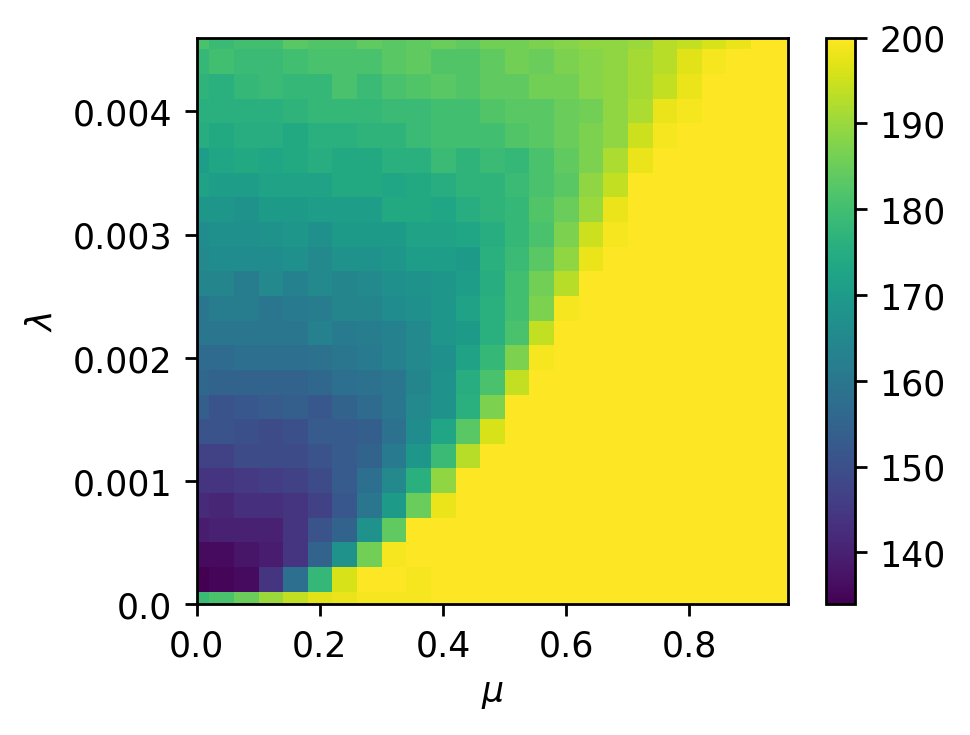}
    \includegraphics[width=0.35\linewidth]{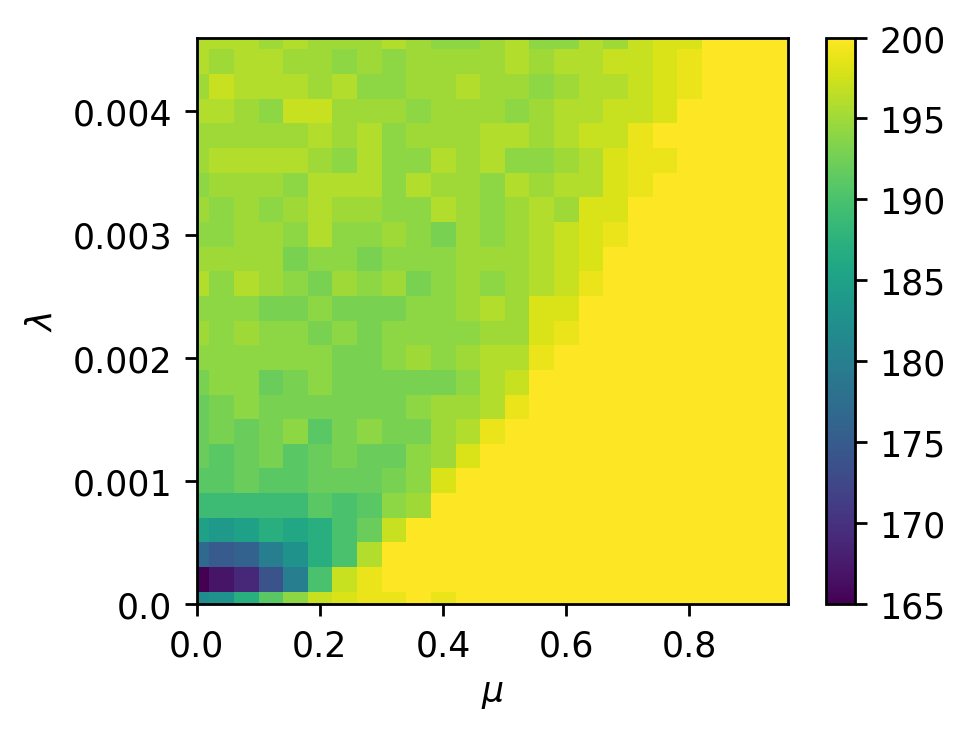}
    \includegraphics[width=0.35\linewidth]{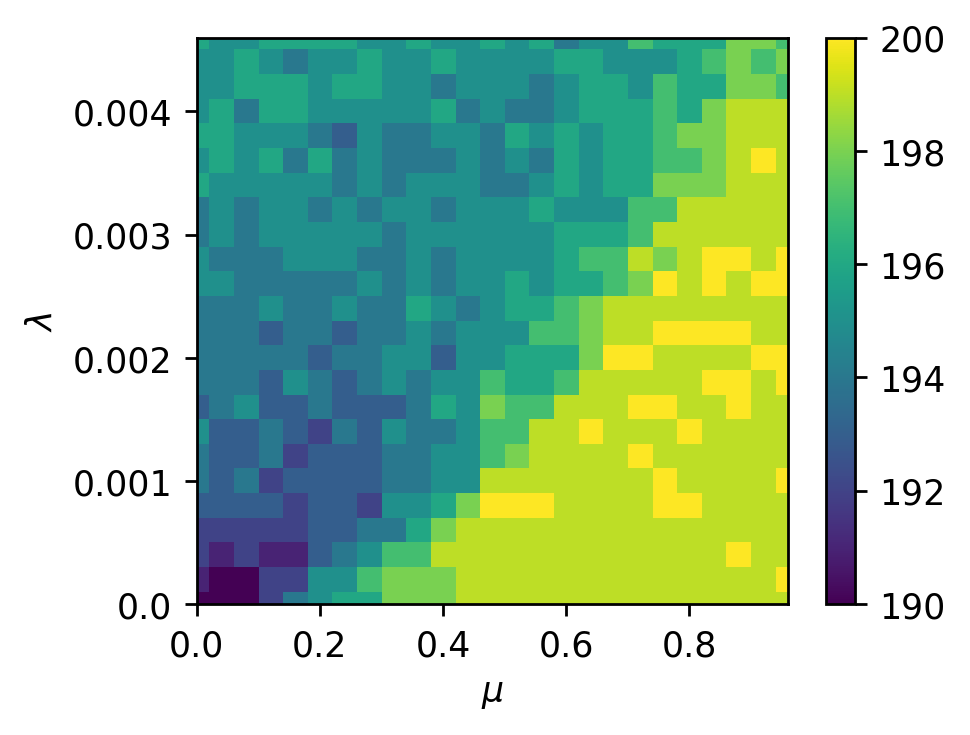}
    \includegraphics[width=0.35\linewidth]{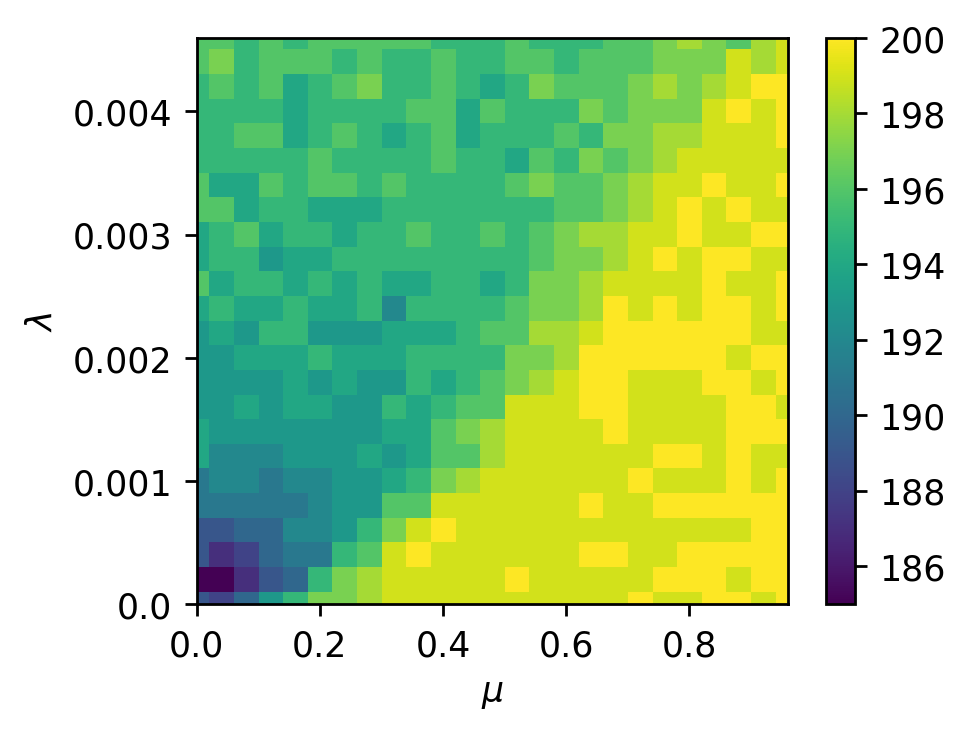}
    \caption{Rank of the converged solution for two-layer linear (upper left), tanh (upper right), relu (lower left) and swish (lower right) models.}
    \label{fig: nn adam}
\end{figure}

\subsection{Experiment with Adam}
We note that the phenomena we studied is not just a special feature of the SGD, but, empirically, seems to be a universal feature of first-order optimization methods that rely on minibatch sampling. Here, we repeat the experiment in Figure~\ref{fig: more phase diagrams}. We train with the same data and training procedure, except that we replace SGD with Adam \cite{journals/corr/KingmaB14_adam}, the most popular first-order optimization method in deep learning. Figure~\ref{fig: nn adam} shows that similarly to SGD, Adam also converges to the low-rank saddles in similar regions of learning rate and $\mu$.

\clearpage

\section{Proofs}
\subsection{Proof of Theorem~\ref{theo: 1d symmetry lyapunov condition}}

\begin{proof}
Consider Eq.~\eqref{eq: sgd dynamics linearized}:
\begin{equation}
    \theta_{t+1} = \theta_t - \lambda \hat{H}(x) (\theta_t - \theta^*).
\end{equation}
Defining $w_t=\theta_t - \theta^*$, this equation can be written as
\begin{equation}\label{eq: proof linear dynamics}
    w_{t+1} = w_t - \lambda \hat{H}(x) w_t,
\end{equation}
which is mathematically equivalent to the case when $\theta^*=0$. Therefore, without loss of generality, we write the dynamics in the form of Eq.~\eqref{eq: proof linear dynamics} in this proof and the rest of the proofs.

Now, when $\hat{H} \propto nn^T$ is rank-$1$, we can multiply $n^T$ from the left on both sides of the dynamics to obtain
\begin{equation}
     n^Tw_{t+1} = n^Tw_t - \lambda h(x) n^Tw_t.
\end{equation}
The dynamics thus becomes one-dimensional in the direction of $n^T$. 

Let $h_t$ denote the eigenvalue of the Hessian of the randomly sampled batch at time step $t$. The dynamics in Eq.~\eqref{eq: 1d symmetry dynamics} implies the following dynamics
\begin{equation}
    \|n^Tw_{t+1}\| / \|n^Tw_t\| = |1 - \lambda h_t|,
\end{equation}
which implies
\begin{equation}
    \|n^Tw_{t+1}\| / \|n^T w_0\| = \prod_{\tau=1}^t |1 - \lambda h_\tau|.
\end{equation}
We can define auxiliary variables $z_t:= \log (\|n^T w_{t+1}\| / \|n^T w_0\|) - m$ and $m := \E [\log (\|n^T w_{t+1}\| / \|n^T w_0\|)] = t\E_x [\log |1 - \lambda h_t|]$. Let $\epsilon>0$. We have that 
\begin{align}
    \mathbb{P}(\|n^T w_t\| < \epsilon)  &= \mathbb{P}(\|n^T w_0\|e^{z_t + m} < \epsilon)\\
    &= \mathbb{P}\left(\frac{1}{{t}} z_t < \frac{1}{{t}} ( \log \epsilon / \|n^T w_0\| - m)\right)\\
    &= \mathbb{P}\left( \frac{z_t}{t} <  - \E_x \log |1 - \lambda h_t| + o(1)\right).
\end{align}
By the law of large numbers, the left-hand side of the inequality converges to $0$, whereas the right-hand side converges to a constant. Thus, we have, for all $\epsilon>0$,
\begin{equation}
    \lim_{t\to \infty} \mathbb{P}(\|n^T w_t\| < \epsilon) = \begin{cases} 1 &\text{if $m<0$};\\
    0 &\text{if $m>1$}.
    \end{cases}
\end{equation}
This completes the proof.
\end{proof}

\subsection{Proof of Proposition~\ref{prop: moment insufficiency}}

\begin{proof}
Part 2 of the proposition follows immediately from Proposition~\ref{prop: no convergence}, which we prove below. Here, we prove part 1. 

It suffices to consider a dataset with two data points for which $h(x_1)=1/\lambda$ and $h(x_2)=c_0$, where each data point is sampled with equal probability. Let $c_0$ be such that 
\begin{equation}
    |1-\lambda c_0|^p > \frac{1}{2}.
\end{equation}
Now, we claim that this dynamics converges to zero in probability. To see this, note that
\begin{equation}
    \|z_{t+1}\| = \begin{cases}
        \|z_{t}\||1 - \lambda/\lambda| =0 &\text{with probability $0.5$};\\
        \|z_{t}\||1 - \lambda c_0| &\text{with probability $0.5$}.
    \end{cases}
\end{equation}
Therefore, at time step $t$, $\p(z_t=0) \geq 1- 2^{-t}$, which converges to $0$. This means that $z_t$ converges in probability to $0$. 

Meanwhile, the $p$-norm is
\begin{align}
    \E[\|z_{t+1}\|^p] &= \frac{1}{2} \E[\|z_{t}\|^p] |1 - \lambda c_0|^p\\
    &\propto \frac{1}{2^t}|1 - \lambda c_0|^{pt} \to 0.
\end{align}
The convergence to zero follows from the construction that $|1-\lambda c_0|^p > \frac{1}{2}$. This completes the proof.
\end{proof}

\begin{proposition}\label{prop: no convergence}
    (No convergence in $L_p$.) For every strict saddle point $\theta^*$, there exists an initialization $\theta_0$ such that for any $\lambda\in \mathbb{R}_+$ and distance function $f(\cdot, \theta^*)$, $\theta^*$ is unstable in $f$.
\end{proposition}
\begin{proof}
    This problem is easy to prove when $\theta$ is one-dimensional. For a high-dimensional $\theta$, the dynamics of SGD is
    \begin{equation}
        \theta_{t+1} = (I-\lambda \hat{H}_t)\theta_t.
    \end{equation}
Note that the expected value of $\theta_t$ is the same as the gradient descent iterations:
\begin{equation}
    \E[\theta_{t+1}] = (I-\lambda \E[\hat{H}])\E[\theta_t] = (I-\lambda \E[\hat{H}])^t \theta_0,
\end{equation}
which diverges if $\theta_0$ is in one of the escape directions of $\E[\hat{H}]$, which exist by the definition of strict saddle points. 

Taking the $f-$distance of both sides and taking expectation, we obtain
\begin{align}
    \E[f(\theta_{t}, \theta^*)] & \geq f(\E[\theta_t], \theta^*)\\
    &= f\left((I-\lambda \E[\hat{H}])^t\theta_0, \theta^*\right)\not\to 0.
\end{align}
The first line follows from the fact that the distance function is convex by definition, and so one can apply Jensen's inequality.

Therefore, as long as $\theta_0$ overlaps with the concave directions of $\E[\hat{H}]$, the argument of $f$ diverges, which implies that the distance function converges to a nonzero value. The expected value of $\theta_t$ is just the gradient descent trajectory, which diverges for any strict saddle point. 

By definition, $\E[\hat{H}]$ contains at least one negative eigenvalue, and so the directions that do not overlap with this direction are strict linear subspaces with dimensions lower than the the total available dimensions. This is a space with Lesbegue measure zero. The proof is complete.
\end{proof}

\subsection{Proof of Theorem~\ref{theo: main theorem}}
Let us first state the Furstenberg-Kesten theorem.
\begin{theorem} (Furstenberg-Kesten theorem)
    Let $X_1,\ X_2,\ X_3,...$ be independent random square matrices drawn from a metrically transitive time-independent stochastic process and $\E[\log_+\|X^1\| < \infty]$, then\footnote{$\log_+(x) =\max(\log x,0)$.}
    \begin{equation}
        \lim_{n\to \infty} \frac{1}{n}\log\|X_1 X_2...X_n\|= \lim_{n\to \infty}\E \left[\frac{1}{n}\log\|X_1 X_2...X_n\|\right]
    \end{equation}
    with probability $1$, where $\|\cdot\|$ denotes any matrix norm.
\end{theorem}
Namely, the Lyapunov exponent of every trajectory converges to the expected value almost surely. Essentially, this is a law of large numbers for the Lyapunov exponent. 

Now, we present the proof of Theorem~\ref{theo: main theorem}.

\begin{proof}
First of all, we define $m_t= \log \|\theta_t - \theta^*\|$ and $z_t = m_t - \E[m_t]$. By definition, we have
\begin{align}
    \mathbb{P}( g_t < \epsilon)  &= \mathbb{P}(e^{z_t + m_t} < \epsilon)\\
    &= \mathbb{P}\left(\frac{1}{{t}} (z_t + \E[m_t]) < \frac{1}{{t}} \log \epsilon \right)\\
    &= \mathbb{P}\left( \frac{1}{{t}} (z_t + \E[m_t])  <  o(1)\right).\label{eq: proof prob 1}
\end{align}

We can lower bound this probability by
\begin{equation}
    \mathbb{P}\left( \frac{1}{{t}} (z_t + \E[m_t])  < o(1)\right) \geq \mathbb{P}\left( \frac{1}{{t}} \max_{\theta_0}(z_t + \E[m_t])  < o(1)\right).
\end{equation}

By the definition of SGD, we have 
\begin{align}
    \frac{1}{t} \max_{\theta_0}\left(z_t(\theta_0) + \E[m_t]\right) &= \frac{1}{t} \max_{\theta_0}\log \left\| \prod_i^t(I - \lambda \hat{H}_i) (\theta_t -\theta_0) \right\|.
\end{align}
By the Furstenberg-Kesten theorem~\cite{furstenberg1960products}, this quantity converges to the constant $\Lambda = \lim_{t\to\infty}\E[m_t]/t \in \mathbb{R}$ almost surely. Namely, $z_t/t$ converges to $0$ for almost every SGD trajectory.

Thus, for every $\epsilon$, if $\Lambda <0$, Eq.~\eqref{eq: proof prob 1} can be bounded as
\begin{equation}
    \lim_{t\to \infty}\mathbb{P}( g_t < \epsilon) = \mathbb{P}(\Lambda<0) = 1.
\end{equation}
Because $\Lambda$ is a constant, we have that if $\Lambda <0$, all trajectories from all initialization converge to $0$. This finishes the first part of the proof. For the second part, simply let $z_t$ be the trajectory starting from the trajectory that achieves the maximum Lyapunov exponent. Again, this dynamics escapes with probability $1$ by the Furstenberg-Kesten theorem. The proof is complete.
\end{proof}

\subsection{Proof of Theorem~\ref{theo: type II saddle at init}}
\label{ss:proof type II init}

\begin{proof}
With the linear approximation of activation function $\sigma(x)\approx c_0 x$, the neural network takes the forms of $f(x) = c_0^D\prod_{i=1}^D W^{(i)}x + \sum_{i=1}^{D-1} c_0^{(D-i)} \prod_{j=i+1}^D W^{(j)} b^{(i)}$. The gradient of the loss $\ell(f(x), y(x))$ is thus
\begin{align}
    \nabla_\theta \ell(f(x), y(x)) = (\nabla_{f(x)}\ell)^T  \nabla_{\theta} f(x).
\end{align}
For layer $i$, the gradient
\begin{align}
    \nabla_{W^{(i)}} f(x) & \propto c_0^D\left(\prod_{j=1}^{i-1}W^{(j)}x\right)\left(\prod_{j=i+1}^{D}W^{(j)}\right) + \sum_{k=1}^{i-1} c_0^{(D-k)} \left(\prod_{j=k+1}^{i-1}W^{(j)}b^{(k)}\right)\left(\prod_{j=i+1}^{D}W^{(j)}\right);\\
    \nabla_{b^{(i)}} f(x) & \propto c_0^{(D-i)} \prod_{j=i+1}^D W^{(j)}.
\end{align}
As there is no constant term in the gradient, the gradient vanishes at $\theta = 0$ for all $x$, and so is the projection of the gradient. Thus, $\theta=0$ is a type-II saddle, under the assumption that $\theta=0$ is not a local minimum. 
\end{proof}

\subsection{Proof of Proposition~\ref{prop: phase diagram}}
\begin{proof}
    We consider the dynamics of SGD around a saddle:
\begin{equation}
    \ell = -\chi \sum_i u_i w_i,
\end{equation}
where we have combined ${\frac{1}{S}\sum_{(x, y)\in B}}xy$ into a single variable $\chi$. The dynamics of SGD
is 
\begin{equation}
\begin{cases}
    w_{i,t+1} = w_{i,t} +\lambda \chi u_{i,t};\\
    u_{i,t+1} = u_{i,t} +\lambda \chi w_{i,t}.
\end{cases}
\end{equation}
Namely, we obtain a set of coupled stochastic difference equations. Since the dynamics is the same for all values of the index $i$, we omit $i$ from now on. This dynamics can be decoupled if we consider two transformed parameters: $h_t = w_t + u_t$ and $m_t = w_t -u_t$. The dynamics for these two variables is given by
\begin{equation}
\begin{cases}
    h_{t+1} = h_{t} +\lambda \chi h_{t};\\
    m_{t+1} = m_{t} -\lambda  \chi m_{t}.
\end{cases}
\end{equation}
We have thus obtained two decoupled linear dynamics that take the same form as that in Theorem~\ref{theo: 1d symmetry lyapunov condition}. Therefore, as immediate corollaries, we know that $h$ converges to $0$ if and only if $\E_{B}[\log|1+\lambda \chi|] <0$, and $m$ converges to $0$ if and only if $\E_{B}[\log |1- \lambda \chi|] <0$.

When both $h$ and $m$ converge to zero in probability, we have that both $w$ and $u$ converge to zero in probability. For the data distribution under consideration {in section \ref{sec: phases of learning sgd} and for batch size one}, we have
\begin{equation}
    \E[\log|1+\lambda \chi|] = \frac{1}{2} \log\left|(1+\lambda)(1+\lambda a)\right|
\end{equation}
and 
\begin{equation}
    \E[\log|1-\lambda \chi|] = \frac{1}{2} \log\left|({1-\lambda})({1-\lambda a})\right|.
\end{equation}
There are four cases: (1) both conditions are satisfied; (2) one of the two is satisfied; (3) neither is satisfied. These correspond to four different phases of SGD around this saddle.
\end{proof}

\end{document}